\documentclass[twoside,11pt]{article}

\usepackage{jmlr2e}

\newcommand{\commentout}[1]{}

\usepackage{amsmath,amssymb,alltt}
\usepackage{bold-extra} 
\usepackage{bm,hyperref,subfigure,wasysym,subfigure,graphicx,enumitem}

\newcommand\tabelT{\rule{0pt}{3ex}}
\newcommand\tabelB{\rule[-1.2ex]{0pt}{0pt}}

\let\phi\varphi
\newcommand{\tn}[1]{\textnormal{#1}}

\newcommand{\Var}{\mathop{\tn{Var}}}
\newcommand{\bpref}[1]{mix loss property~{\#}\ref{#1}}
\newcommand{\regret}{\mathcal{R}}
\newcommand{\bound}{\mathcal{G}}
\newcommand{\Rah}{\regret^{\tn{ah}}}
\newcommand{\Rff}{\regret^{\tn{ff}}}
\newcommand{\Rftl}{\regret^{\tn{ftl}}}
\newcommand{\Rahprime}{\regret^{\tn{ah}'}}
\newcommand{\Rffprime}{\regret^{\tn{ff}'}}
\newcommand{\Rftlprime}{\regret^{\tn{ftl}'}}
\newcommand{\Hff}{H^{\tn{ff}}}

\newcommand{\Dah}{\Delta^{\tn{ah}}}
\newcommand{\dah}{\delta^{\tn{ah}}}

\newcommand{\Hah}{H^{\tn{ah}}}
\newcommand{\hah}{h^{\tn{ah}}}
\newcommand{\Vah}{V^{\tn{ah}}}
\newcommand{\vah}{v^{\tn{ah}}}
\newcommand{\A}{\underline{R}}
\newcommand{\F}{\overline{R}}
\newcommand{\eah}{\eta^\tn{ah}}
\newcommand{\half}{\tfrac{1}{2}}
\newcommand{\loss}{\ell}

\newcommand{\Bah}{M^{\tn{ah}}}

\newcommand{\Bw}{\underline{M}}
\newcommand{\Bl}{\overline{M}}
\newcommand{\Dw}{\underline{\Delta}}
\newcommand{\Dl}{\overline{\Delta}}
\newcommand{\Vw}{\underline{V}}

\DeclareBoldMathCommand{\L}{L}
\DeclareBoldMathCommand{\G}{G}
\DeclareBoldMathCommand{\g}{g}
\DeclareBoldMathCommand{\vloss}{\ell}      
\DeclareBoldMathCommand{\w}{w}             
\DeclareBoldMathCommand{\dot}{\cdot}       

\usepackage{color}
\usepackage{index} 
\newindex{todo}{tod}{tnd}{TODO List} 
\newindex{fixme}{fix}{fnd}{FIXME List} 

\jmlrheading{1}{2013}{1-48}{4/00}{10/00}{Steven de Rooij, Tim van Erven,
Peter D. Gr\"unwald and Wouter M. Koolen}

\ShortHeadings{Follow the Leader If You Can, Hedge If You Must}{De Rooij, Van Erven, Gr\"unwald and
Koolen}
\firstpageno{1}

\title{Follow the Leader If You Can, Hedge If You Must}
\author{%
\name Steven de Rooij \email s.de.rooij@cwi.nl\\
\addr Centrum Wiskunde \& Informatica (CWI)\\
Science Park 123, P.O. Box 94079\\
1090 GB Amsterdam, the Netherlands
\AND
\name Tim van Erven \email tim@timvanerven.nl\\
\addr D\'epartement de Math\'ematiques\\
Universit\'e Paris-Sud\\
91405 Orsay Cedex, France
\AND
\name Peter D. Gr\"unwald \email pdg@cwi.nl\\
\addr Centrum Wiskunde \& Informatica (CWI) and Leiden University\\
Science Park 123, P.O. Box 94079\\
1090 GB Amsterdam, the Netherlands
\AND
\name Wouter M. Koolen \email wouter@cs.rhul.ac.uk\\
\addr Department of Computer Science\\
Royal Holloway, University of London\\
Egham Hill, Egham, Surrey\\
TW20 0EX, United Kingdom%
}

\editor{}

\begin{document}

\maketitle

\begin{abstract}%
  Follow-the-Leader (FTL) is an intuitive sequential prediction
  strategy that guarantees constant regret in the stochastic setting,
  but has terrible performance for worst-case data. Other hedging
  strategies have better worst-case guarantees but may perform much
  worse than FTL if the data are not maximally adversarial. We introduce
  the FlipFlop algorithm, which is the first method that provably
  combines the best of both worlds.

  As part of our construction, we develop AdaHedge, which is a new way
  of dynamically tuning the learning rate in Hedge without using the doubling
  trick.  AdaHedge refines a method by
  \citet*{CesaBianchiMansourStoltz2007}, yielding slightly improved
  worst-case guarantees.

  By interleaving AdaHedge and FTL, the FlipFlop algorithm achieves
  regret within a constant factor of the FTL regret, without
  sacrificing AdaHedge's worst-case guarantees.
   
  AdaHedge and FlipFlop do not need to know the range of the losses in
  advance; moreover, unlike earlier methods, both
  have the intuitive property that the issued weights are invariant
  under rescaling and translation of the losses. The losses are also
  allowed to be negative, in which case they may be interpreted as
  gains.


\end{abstract}


\section{Introduction}
We consider sequential prediction in the general framework of Decision
Theoretic Online Learning (DTOL) or ``the Hedge setting''
\citep{FreundSchapire1997}, which is a variant of ``prediction with expert
advice'' \citep{Vovk1998}. Our goal is to develop a sequential
prediction algorithm that performs well not only on adversarial data,
which is the scenario most studies worry about, but also when the data
are easy, as is often the case in practice. Specifically, with
adversarial data, the worst-case regret (defined below) for any
algorithm is $\Omega(\sqrt{T})$, where $T$ is the number of
predictions to be made. Algorithms such as Hedge, which have been
designed to achieve this lower bound, typically continue to suffer
regret of order $\sqrt{T}$, even for easy data, where the regret of
the more intuitive but less robust Follow-the-Leader (FTL) algorithm
(also defined below) is \emph{bounded}. Here, we present the first
algorithm which, up to constant factors, provably achieves both the
regret lower bound in the worst case, \emph{and} a regret not exceeding
that of FTL. Below, we first describe the Hedge setting. Then we
introduce FTL, discuss sophisticated versions of Hedge from the
literature, and give an overview of the results and contents of this
paper.

\subsection{Overview}
In the hedge setting, a learner has to decide each round
$t=1,2,\ldots$ on a weight vector $\w_t = (w_{t,1},\ldots,w_{t,K})$
over $K$ ``experts''.  (This term derives from the strongly related
prediction with expert advice paradigm
\citep{LittlestoneWarmuth1994,Vovk1998,CesaBianchiLugosi2006}.)
Nature then reveals a $K$-dimensional vector containing the losses of
the experts $\vloss_t = (\loss_{t,1},\ldots,\loss_{t,K})\in{\mathbb
  R}^K$. Learner's loss is the dot product $h_t=\w_t\dot\vloss_t$,
which can be interpreted as the expected loss if Learner uses a mixed
strategy and chooses expert $k$ with probability $w_{t,k}$.  We denote
cumulative versions of a quantity by capital letters, and vectors are
in bold face. Thus, $L_{T,k}=\sum_{t=1}^T\loss_{t,k}$ denotes the
cumulative loss of expert $k$ up to the present round $T$, and
$H_T=\sum_{t=1}^T h_t$ is Learner's cumulative loss (the ``Hedge
loss'').

Learner's performance is evaluated in terms of her \emph{regret},
which is the difference between her cumulative loss and the cumulative
loss of the best expert:
\[
\regret_T=H_T-L^*_T,\qquad\text{where $L^*_T=\min_k L_{T,k}$.}
\]

A simple and intuitive strategy for the Hedge setting is
Follow-the-Leader (FTL), which puts all weight on the expert(s) with the
smallest loss so far. More precisely, we will define the weights $\w_t$
for FTL to be uniform on the set of leaders $\{k\mid L_{t-1,k}=L^*_{t-1}\}$, which
is often just a singleton. FTL works very well under many circumstances,
for example in stochastic scenarios where the losses are independent and
identically distributed (i.i.d.). In particular, the regret for
Follow-the-Leader is bounded by the number of times the leader is
overtaken by another expert (Lemma~\ref{lem:ftl_regret}), which in the
i.i.d.\ case almost surely happens only a finite number of times (by the
uniform law of large numbers), provided the mean loss of the
best expert is smaller than the mean loss of the other experts. As
demonstrated by the experiments in Section~\ref{sec:experiments}, many
more sophisticated algorithms can perform significantly worse than FTL.

The problem with FTL is that it breaks down badly when the data
are antagonistic. For example, if one out of two experts incurs
losses $\half, 0, 1, 0, \ldots$ while the other incurs opposite losses $0, 1,
0, 1, \ldots$, the regret for FTL is about $T/2$ (this scenario is
further discussed in Section~\ref{sec:exp1}). This has prompted the
development of a multitude of alternative algorithms that provide
better worst-case regret guarantees.

The seminal strategy for the learner is called \emph{Hedge}
\citep{FreundSchapire1997,FreundSchapire1999}. Its performance
crucially depends on a parameter $\eta$ called the \emph{learning
  rate}. Hedge can be interpreted as a generalisation of FTL, which is
recovered in the limit for $\eta\to\infty$. In many analyses, the
learning rate is changed from infinity to a lower value that optimizes
some upper bound on the regret. Doing so requires precognition of the
number of rounds of the game, or of some property of the data such as
the eventual loss of the best expert $L^*_T$. The simplest way to
address this issue is 
to use the so-called \emph{doubling trick}: setting a budget
on the relevant statistic, and restarting the algorithm with a double
budget when the budget is depleted
\citep{CesaBianchiLugosi2006,CFHHSW1997,HazanKale2008}; $\eta$ can
then be optimised for each individual block in terms of the
budget. Better bounds, but harder analyses, are typically obtained if
the learning rate is adjusted each round based on previous
observations, see e.g.\
\citep{CesaBianchiLugosi2006,AuerCesaBianchiGentile2002}.

The Hedge strategy presented by \citet*{CesaBianchiMansourStoltz2007}
is very closely related to the approach described here. The relevant
algorithm, which we refer to as \textrm{CBMS}, is defined in (16) in
Section~4.2 of their paper. Its regret satisfies\footnote{The leading
  constant of 4 was later improved to approximately 2.63 in
  ~\citep[Remark~2.2]{Gerchinovitz2011}, essentially by using
  Lemma~\ref{lem:bayesbound} below. Our approach allows a further
  reduction to 2.}
\begin{equation}\label{eq:stoltz}
\regret_T^{\textrm{CBMS}}\le
4\sqrt{\frac{L_T^*(\sigma T-L_T^*)}{T}\ln K}+39\sigma \max\{1,\ln K\},
\end{equation}
where $\sigma$ is the range of observed losses; if all losses are
nonnegative, this is the maximum loss attained by any expert at any
time.  Thus, in the worst case this algorithm has a regret of order
$\sqrt{T}$, but it performs much better when the loss of the best
expert $L^*_T$ is close to either 0 or $\sigma T$.

The goal of this work is to develop a strategy that retains this
worst-case bound, but has even better guarantees for easy data: its
performance should never be substantially worse than that of
Follow-the-Leader. At first glance, this may seem like a trivial
problem: simply take both FTL and some other hedging strategy with
good worst-case guarantees, and combine the two by using FTL or Hedge
recursively. To see why such approaches do not work, suppose that FTL
achieves regret $\Rftl_T$, while the safe hedging strategy achieves
regret $\regret_T^\tn{safe}$. We would only be able to prove that the
regret of the combined strategy compared to the best original expert
satisfies $\regret_T^\tn{c}\le
\min\{\Rftl_T,\regret_T^\tn{safe}\}+\bound_T^\tn{c}$, where
$\bound_T^\tn{c}$ is the worst-case regret guarantee for the
combination method, e.g.~\eqref{eq:stoltz}. In general, either
$\Rftl_T$ or $\regret_T^\tn{safe}$ may be close to zero, while at the
same time both algorithms have loss close to $T/2$, so that
$\bound_T^\tn{c}=\Omega(\sqrt{T})$. That is, the overhead of the
combination method will dominate the regret!

We address this issue in two stages. First, in
Section~\ref{sec:adahedge}, we develop AdaHedge, which is a refinement
of the CBMS strategy of~\citet{CesaBianchiMansourStoltz2007} for which
we can obtain similar bounds, including~\eqref{eq:stoltz}, but with a
factor $2$ improvement of the dominant term
(Theorem~\ref{thm:adahedge_regret}). Like CMBS, the learning rate is
tuned in terms of a direct measure of past performance. However,
AdaHedge not only recovers the ``fundamental'' regret \emph{bounds} of
CMBS, but it has the intuitive property that the weights it issues are
themselves invariant to translation and rescaling of the losses (see
Section~\ref{sec:invariance}). The analysis of AdaHedge is also
surprisingly clean. A preliminary version of this strategy was
presented at NIPS \Citep{ErvenGrunwaldKoolenDeRooij2011}.

Second, in Section~\ref{sec:flipflop}, we build on AdaHedge to develop
the FlipFlop approach, which alternates between FTL and AdaHedge. For
this strategy we can guarantee
\[
\Rff_T=O(\min\{\regret_T^{\tn{ftl}},\bound_T^\tn{ah}\}),
\]
where $\bound_T^\tn{ah}$ is the regret guarantee for AdaHedge;
Theorem~\ref{thm:flipflop_regret} provides a precise statement. Thus,
FlipFlop is the first algorithm that provably combines the benefits of
Follow-the-Leader with robust behaviour for antagonistic data.

A key concept in the design and analysis of our algorithms is what we
call the {\em mixability gap}, introduced in
Section~\ref{sec:mixgap}. This quantity also appears in earlier works,
and seems to be of fundamental importance in both the current Hedge
setting as in stochastic settings. We elaborate on this in
Section~\ref{sec:big} where we provide the big picture underlying this
research and we briefly indicate how it relates to practical work such
as \citep{DevaineGGS12}.

\subsection{Related Work}\label{sec:related}
As mentioned, AdaHedge is a refinement of the strategy analysed
by~\cite{CesaBianchiMansourStoltz2007}, which is itself more
sophisticated than most earlier approaches, with two notable
exceptions. First, by slightly modifying the weights, and tuning the
learning rate in terms of the cumulative empirical variance of the
best expert, \citet{HazanKale2008} are able to obtain a
bound that multiplicatively dominates~\eqref{eq:stoltz}.
However, their method requires the doubling trick, and as demonstrated
by the experiments in Section~\ref{sec:experiments}, it does not
achieve the benefits of FTL. Second, Chaudhuri, Freund and Hsu
(\citeyear{ChaudhuriFreundHsu2009}) describe a strategy called
NormalHedge that can efficiently compete with the best
$\epsilon$-quantile of experts; their bound is incomparable with the
bound for AdaHedge. In the experimental section we discuss the
performance of these approaches compared to AdaHedge and FlipFlop.

Other approaches to sequential prediction include defensive forecasting
\citep{VovkTakemuraShafer2005}, and Following the Perturbed Leader
\citep{KalaiVempala2003}. These radically different approaches also
allow competing with the best $\epsilon$-quantile, see
\citep{ChernovVovk2010} and~\citep{HutterPoland2005}; the latter article
also considers nonuniform weights on the experts.

The ``safe MDL'' and ``safe Bayesian'' algorithms by
\citet{Grunwald2011,Grunwald2012} share the present work's focus on
the mixability gap as a crucial part of the analysis, but are
concerned with the stochastic setting where losses are not adversarial
but i.i.d. FlipFlop, safe MDL and safe Bayes can all be interpreted as
methods that attempt to choose a learning rate $\eta$ that keeps the
mixability gap small (or, equivalently, that keep the Bayesian
posterior or Hedge weights ``concentrated'').





\subsection{Outline}

In the next section we present and analyse AdaHedge. Then, in
Section~\ref{sec:flipflop}, we build on AdaHedge to develop the
FlipFlop strategy. The analysis closely parallels that of AdaHedge,
but with extra complications at each of the steps. Both algorithms are
initially analysed for normalised losses, which take values in the
interval $[0,1]$. In Section~\ref{sec:invariance} we extend their
analysis to unnormalised losses. Then we compare AdaHedge and FlipFlop
to existing methods in experiments with artificial data in
Section~\ref{sec:experiments}. Finally, Section~\ref{sec:discussion}
contains a discussion, with ambitious suggestions for future work.

\section{AdaHedge}\label{sec:adahedge}

In this section, we present and analyse the AdaHedge strategy.  The
behaviour of AdaHedge does not change under scaling or translation of
the losses. However, to keep the analysis simple, we will initially
assume throughout this and the next section that all losses are
normalised to the unit interval, i.e.\ $\vloss_t\in[0,1]^K$.
Unnormalised losses are treated in Section~\ref{sec:invariance}, by a
reduction to the normalised case.

To introduce our notation and proof strategy, we start with the
simplest possible analysis of vanilla Hedge, and then move on to
refine it for AdaHedge.

\subsection{Basic Hedge Analysis for Constant Learning Rate}\label{sec:mixgap}

Following \citet{FreundSchapire1997} we define the \emph{Hedge} or \emph{exponential weights} strategy as the
choice of weights 
\begin{equation}\label{eq:weights}
  w_{t,k} = \frac{w_{1,k} e^{-\eta L_{t-1,k}}}{Z_t},
\end{equation}
where $\w_1 = (1/K,\ldots,1/K)$ is the uniform distribution, $Z_t = \w_1
\dot e^{-\eta \L_{t-1}}$ is a normalizing constant, and $\eta \in
(0,\infty)$ is a parameter of the algorithm called the \emph{learning
  rate}. If $\eta = 1$ and one imagines $L_{t-1,k}$ to be the negative
log-likelihood of a sequence of observations, then $w_{t,k}$ is the
Bayesian posterior probability of expert $k$ and $Z_t$ is the marginal
likelihood of the observations. Consequently, like in Bayesian
inference, the weights can be updated multiplicatively, i.e.\ we have
$w_{t+1,k} \propto w_{t,k}e^{-\eta \loss_{t,k}}$.

The loss incurred by Hedge in round $t$ is $h_t=\w_t\dot\vloss_t$, and
our goal is to obtain a good bound on the cumulative Hedge loss
$H_T=\sum_{t=1}^T h_t$. To this end, it turns out to be technically
convenient to approximate $h_t$ by the \emph{mix loss}
%
\begin{equation}\label{eq:mixloss}
  m_t = -\frac{1}{\eta}\ln(\w_t\dot e^{-\eta \vloss_t})
\end{equation}
which accumulates to $M_T = \sum_{t=1}^T m_t$. This approximation is a
standard tool in the literature. For example, the mix loss $m_t$
corresponds to the loss of Vovk's (\citeyear{Vovk1998, Vovk2001})
Aggregating Pseudo Algorithm, and tracking the evolution of $-m_t$ is a
crucial ingredient in the proof of Theorem~2.2 of
\citet{CesaBianchiLugosi2006}.

The definitions of Hedge and the mix loss may both be extended to
$\eta = \infty$ by letting $\eta$ tend to $\infty$. In the case of
Hedge, we then find that $\w_t$ becomes a uniform distribution on the
set of experts $\{k \mid L_{t-1,k} = L^*_{t-1}\}$ that have incurred smallest
cumulative loss before time $t$. That is, Hedge with $\eta = \infty$
reduces to \emph{Follow-the-Leader}, with ties broken by dividing the
probability mass uniformly. For the mix loss, we find that the limiting
case as $\eta$ tends to $\infty$ is $m_t = L_t^* - L_{t-1}^*$.

In our approximation of the Hedge loss $h_t$ by the mix loss $m_t$, we
call the approximation error $\delta_t=h_t-m_t$ the \emph{mixability
gap}. Bounding this quantity is a standard part of the analysis of
Hedge-type algorithms (see, for example, Lemma~4 of
\citet{CesaBianchiMansourStoltz2007}) and it also appears to be a
fundamental notion in sequential prediction even when only so-called
mixable losses are considered \citep{Grunwald2011,Grunwald2012}; see
also Section~\ref{sec:big}. We let $\Delta_T=\delta_1+\ldots+\delta_T$
denote the cumulative mixability gap, so that the regret for Hedge may
be decomposed as
\begin{equation}\label{eq:hedge_regret}
\regret_T=H_T-L_T^*=M_T-L_T^*+\Delta_T.
\end{equation}
Here $M_T - L_T^*$ may be thought of as the regret under the mix loss
and $\Delta_T$ is the cumulative approximation error when approximating
the Hedge loss by the mix loss. Throughout the paper, our proof strategy
will be to analyse these two contributions to the regret, $M_T-L_T^*$
and $\Delta_T$, separately. The following lemma, which is proved in
Appendix~\ref{proof:basic_properties}, collects a few basic properties:
\begin{lemma}[Mix Loss with Constant Learning Rate]\label{lem:basic_properties} 
  For any learning rate $\eta \in (0,\infty]$
  \begin{enumerate}
  \item\label{it:EB_bound} Mix loss is less than Hedge loss ($m_t \le
    h_t$) so that $\delta_t \geq 0$. Moreover, for losses in the range
    $[0,1]$, we have $m_t \geq 0$ and $h_t \leq 1$, so that also $\delta_t
    \leq 1$.
  \item\label{it:cumbayes} Cumulative mix loss telescopes:
    $M_T=-\frac{1}{\eta}\ln\left(\w_1\dot
      e^{-\eta\L_T}\right)$.
  \item\label{it:bayesbound} Cumulative mix loss approximates the loss of the best
  expert: $\displaystyle L^*_T\le M_T\le L^*_T+\frac{\ln K}{\eta}$.
  \item\label{it:nondecreasing} The cumulative mix loss $M_T$
  is nonincreasing in $\eta$.
  \end{enumerate}
\end{lemma}
In order to obtain a bound for Hedge, one can use the following
well-known bound on the mixability gap, which is obtained using
Hoeffding's bound on the cumulant generating function
\citep[Lemma~A.1]{CesaBianchiLugosi2006}:
\begin{equation}\label{eq:eta_over_eight}
\delta_t\le\frac{\eta}{8},
\end{equation}
from which $\Delta_T\le T\eta/8$. Together with the bound $M_T-L_T^*\le
\ln(K)/\eta$ from \bpref{it:bayesbound} this leads to
\begin{equation}\label{eq:wcbound}
\regret_T = (M_T-L_T^*) +\Delta_T \le \frac{\ln K}{\eta} + \frac{\eta T}{8}.
\end{equation}
The bound is optimized for $\eta=\sqrt{8\ln(K)/T}$, which equalizes the
two terms. This leads to a bound on the regret of $\sqrt{T\ln(K)/2}$,
matching the lower bound on worst-case regret from the textbook by
\citet[Sections~2.2 and 3.7]{CesaBianchiLugosi2006}. We can use this
tuned learning rate if the time horizon $T$ is known in advance; to deal
with the situation where it it is not, the doubling trick can be used,
at the cost of a worse constant factor in the leading term of the regret
bound.

In the remainder of this section, we introduce the AdaHedge strategy,
and refine the steps of the analysis above to obtain a better
regret bound.

\subsection{AdaHedge Analysis}\label{sec:adahedge_analysis}
 In the previous section, we split the regret
for Hedge into two parts: $M_T-L_T^*$ and $\Delta_T$, and we obtained a
bound for both. The learning rate $\eta$ was then tuned to equalise
these two bounds. The main distinction between AdaHedge and other Hedge
approaches is that AdaHedge does not consider an upper bound on
$\Delta_T$ in order to obtain this balance: instead it aims to equalize $\Delta_T$ and $\ln(K)/\eta$. As the cumulative mixability
gap $\Delta_T$ is monotonically increasing and can be \emph{observed}
on-line, it is possible to adapt the learning rate directly based on
$\Delta_T$.

Perhaps the easiest way to achieve this is by using the doubling
trick: each subsequent block uses half the learning rate of the
previous block, and a new block is started as soon as the observed
cumulative mixability gap $\Delta_T$ exceeds the bound on the mix loss
$\ln(K)/\eta$, which ensures these two quantities are equal at the end
of each block. This is the approach taken in an earlier version of
AdaHedge \Citep{ErvenGrunwaldKoolenDeRooij2011}. However, we can achieve
the same goal much more elegantly, by decreasing the learning rate
with time as follows:

%
\begin{equation}\label{eq:learning_rate}
  \eta_t^\tn{ah}=\frac{\ln K}{\Dah_{t-1}}.
\end{equation}
(Note that $\eta^\tn{ah}_1=\infty$.) The
definitions~\eqref{eq:weights} and~\eqref{eq:mixloss} of the weights
and the mix loss are modified to use this new learning rate:
\begin{equation}\label{eq:adahedge_defs}
w^\tn{ah}_{t,k} = \frac{w^\tn{ah}_{1,k} e^{-\eta^\tn{ah}_t L_{t-1,k}}}{\w_1^\tn{ah} \dot
  e^{-\eta^\tn{ah}_t \L_{t-1}}};\qquad m^\tn{ah}_t =
-\frac{1}{\eta^\tn{ah}_t}\ln(\w^\tn{ah}_t\dot e^{-\eta^\tn{ah}_t \vloss_t}),
\end{equation}
with $\w^\tn{ah}_1 = (1/K, \ldots, 1/K)$.
Note that the multiplicative update rule for the weights no longer
applies when the learning rate varies with $t$; the last three results of Lemma~\ref{lem:basic_properties}
are also no longer valid. Later we will also consider other algorithms
to determine variable learning rates; to avoid confusion the
considered algorithm is always specified in the superscript in our
notation. See Table~\ref{tab:defs} for reference.

From now on, AdaHedge will be defined as the Hedge algorithm with
learning rate defined by~\eqref{eq:learning_rate}. For concreteness, a
{\sc matlab} implementation appears in Figure~\ref{fig:adahedge}.

\begin{table}[t]
\centering
\framebox{\begin{minipage}{37em}
\begin{tabular}{ll}
  $\vloss_t$&Loss vector for time $t$\\
  $L^*_t=\min_k L_{t,k}$&Cumulative loss of the best expert\\
  $\w^\tn{alg}_t=e^{-\eta^\tn{alg}_t\dot\L_{t-1}}/\sum_ke^{-\eta^\tn{alg}_t L_{t-1,k}}$&Weights
  played at time $t$\\
  $h^\tn{alg}_t=\w^\tn{alg}_t\dot\vloss_t$&Hedge loss\\
  $m^\tn{alg}_t=-\frac{1}{\eta^\tn{alg}_t}\ln\left(\w^\tn{alg}_t\dot
    e^{-\eta^\tn{alg}_t\vloss_t}\right)$&Mix loss\\[1.4ex]
  $\delta^\tn{alg}_t=h^\tn{alg}_t-m^\tn{alg}_t$&Mixability gap\\
  $v^\tn{alg}_t=\Var_{k\sim\w^\tn{alg}_t}(\loss_{t,k})$&Loss variance
  at time $t$\\
  $\regret^\tn{alg}_t=H^\tn{alg}_t-L^*_t$&Regret at time $t$
\end{tabular}
\\
A capital letter denotes the cumulative value, e.g.\
$\Delta^\tn{alg}_T=\sum_{t=1}^T\delta^{\tn{alg}}_t$.
\\

The ``alg'' in the superscript refers to the algorithm that defines
the learning rate used at each time step: ``$(\eta)$'' represents
Hedge with fixed learning rate $\eta$; ``ah'' denotes AdaHedge,
defined in~\eqref{eq:learning_rate}; ``ftl'' denotes Follow-the-Leader
($\eta^\tn{ftl}=\infty$), and ``ff'' denotes FlipFlop, defined
in~\eqref{eq:ff_eta}.
\end{minipage}}
\caption{Notation\label{tab:defs}}
\end{table}

Our learning rate is similar to that
of~\citet{CesaBianchiMansourStoltz2007}, but it is always higher, and
as such may exploit easy sequences of losses more
aggressively. Moreover our tuning of the learning rate simplifies the
analysis, leading to tighter results; the essential new technical
ingredients appear as lemmas~\ref{lem:r_le_2delta}
and~\ref{lem:delta_square} below.

We analyse the regret for AdaHedge like we did in the previous section
for a fixed learning rate: we again consider $\Bah_T-L_T^*$ and
$\Dah_T$ separately. This time, both legs of the analysis become
slightly more involved. Luckily, a good bound can still be obtained
with only a small amount of work. First we show that the mix loss is
bounded by the mix loss we would have incurred if we would have used the
final learning rate $\eta_T^\tn{ah}$ all along
\citep[Lemma~3]{KalnishkanVyugin2005}:
\begin{lemma}\label{lem:bayesbound}
  Let $\tn{dec}$ be any strategy for choosing the learning rate such
  that $\eta_1 \geq \eta_2 \geq \dots$ Then the cumulative mix loss for
  $\tn{dec}$ does not exceed the cumulative mix loss for the strategy
  that uses the last learning rate $\eta_T$ from the start:
  $M_T^{\tn{dec}} \leq M_T^{(\eta_T)}$.
\end{lemma}
\begin{proof}Using~\bpref{it:nondecreasing}, we have
\[\sum_{t=1}^T m_t^{\tn{dec}} = \sum_{t=1}^T \left(M_t^{(\eta_t)} -
    M_{t-1}^{(\eta_t)}\right)\le\sum_{t=1}^T\left(M_t^{(\eta_t)}-M_{t-1}^{(\eta_{t-1})}\right)=M^{(\eta_T)}_T,\]
  which was to be shown.
\end{proof}

\noindent
We can now show that the two contributions to the regret are still
balanced.

\begin{lemma}\label{lem:r_le_2delta}
  The AdaHedge regret is $\displaystyle \Rah_T=\Bah_T-L_T^*+\Dah_T\le2\Dah_T$.
\end{lemma}
\begin{proof}
  As $\dah_t\ge0$ for all $t$ (by \bpref{it:EB_bound}), the cumulative
  mixability gap $\Dah_t$ is nondecreasing. Consequently, the AdaHedge
  learning rate $\eta^\tn{ah}_t$ as defined in~\eqref{eq:learning_rate} is
  nonincreasing in $t$. Thus Lemma~\ref{lem:bayesbound} applies to
  $\Bah_T$; together with~\bpref{it:bayesbound}
  and~\eqref{eq:learning_rate} this yields
\[
\Bah_T \le M_T^{(\eah_T)}\le L_T^*+\frac{\ln K}{\eah_T}=L_T^*+\Dah_{T-1}\le L_T^*+\Dah_T.\]
Substitution into the trivial decomposition $\Rah_T=\Bah_T-L_T^*+\Dah_T$ yields the result.
\end{proof}

\begin{figure}[t]
\begin{minipage}[t]{0.45\linewidth}
\footnotesize\begin{alltt}
\textsl{\% Returns the losses of AdaHedge.}
\textsl{\% l(t,k) is the loss of expert k at time t}
function h = adahedge(l)
    [T, K]  = size(l);
    h       = nan(T,1);
    L       = zeros(1,K);
    Delta   = 0; 

    for t = 1:T
        eta = log(K)/Delta;
        [w, Mprev] = mix(eta, L);
        h(t) = w * l(t,:)';
        L = L + l(t,:);
        [~, M] = mix(eta, L);
        delta = max(0, h(t)-(M-Mprev));
        \textsl{\% (\textup{max} clips numeric Jensen violation)}
        Delta = Delta + delta;
    end
end
\end{alltt}
\end{minipage}
\hfill%
\begin{minipage}[t]{0.45\linewidth}
\footnotesize\begin{alltt}
\textsl{\% Returns the posterior weights and mix loss}
\textsl{\% for learning rate eta and cumulative loss}
\textsl{\% vector L, avoiding numerical instability.}
function [w, M] = mix(eta, L)
    mn = min(L);
    if (eta == Inf) \textsl{\% Limit behaviour: FTL}
        w = L==mn;
    else
        w = exp(-eta .* (L-mn));
    end
    s = sum(w);
    w = w / s;
    M = mn - log(s/length(L))/eta;
end
\end{alltt}
\end{minipage}
\caption{Numerically robust {\sc matlab} implementation of AdaHedge\label{fig:adahedge}}
\end{figure}

The remaining task is to establish a bound on $\Dah_T$. As before, we
start with a bound on the mixability gap in a single round, but rather
than~\eqref{eq:eta_over_eight}, we use Bernstein's bound on the
mixability gap in a single round to obtain a result that is expressed
in terms of the variance of the losses,
$v^\tn{ah}_t=\Var_{k\sim\w^\tn{ah}_t}[\loss_{t,k}] = \sum_k
w^\tn{ah}_{t,k} (\loss_{t,k} - h^\tn{ah}_t)^2$.

\begin{lemma}[Bernstein's Bound]\label{lem:bernstein}
  Let $\eta_t=\eta^\tn{alg}_t\in(0,\infty)$ denote the finite learning
  rate chosen for round $t$ by any algorithm ``\tn{alg}''. For losses in
  the range $[0,1]$, the mixability gap $\delta^\tn{alg}_t$ satisfies
  \begin{equation}\label{eq:bernstein1}
  \delta^\tn{alg}_t\le \frac{e^{\eta_t}-\eta_t-1}{\eta_t} v^\tn{alg}_t
  \end{equation}
  Further, $v^\tn{alg}_t\le h^\tn{alg}_t(1-h^\tn{alg}_t)\leq 1/4$.
\end{lemma}
\begin{proof}
  This is Bernstein's bound \citep[Lemma~A.5]{CesaBianchiLugosi2006} on
  the cumulative generating function, applied to the random variable
  $\loss_{t,k}$ with $k$ distributed according to $\w^\tn{alg}_t$.
\end{proof}

Bernstein's bound is more sophisticated than~\eqref{eq:eta_over_eight},
because it expresses that the mixability gap $\delta_t$ is small not
only when $\eta_t$ is small, but also when all experts have
approximately the same loss, or when the weights $\w_{t}$ are
concentrated on a single expert.

The next step is to use Bernstein's inequality to obtain a bound on
the cumulative mixability gap $\Dah_T$. In the analysis
of~\citet{CesaBianchiMansourStoltz2007} this is achieved by first
applying Bernstein's bound for each individual round, and then using a
telescoping argument to obtain a bound on the sum. With our learning
rate~\eqref{eq:learning_rate} it is convenient to reverse these steps:
we first telescope, which can now be done with equality, and
subsequently apply a stricter version of Bernstein's inequality.

\begin{lemma}\label{lem:delta_square}
For losses in the range $[0,1]$, AdaHedge's cumulative mixability gap satisfies 
\[
\big(\Dah_T\big)^2\!\le \Vah_T\ln K+(1+\tfrac{2}{3}\ln
K)\Dah_T.
\]
\end{lemma}
\begin{proof}
  In this proof we will omit the superscript ``ah''. Using the definition
  of the learning rate \eqref{eq:learning_rate} and $\delta_t\le1$
  (from~\bpref{it:EB_bound}), we get
\begin{equation}\label{eq:delta_square}\begin{split}
  \Delta_T^2
    &= \sum_{t=1}^T\left(\Delta_t^2 -
    \Delta_{t-1}^2\right)
    = \sum_t\left( (\Delta_{t-1} + \delta_t)^2 -
    \Delta_{t-1}^2\right)
    = \sum_t\left(2\delta_t\Delta_{t-1}+\delta_t^2\right)\\
    &=\sum_t\left(2\delta_t\frac{\ln K}{\eta_t}+\delta_t^2\right)
    \le\sum_t\left(2\delta_t\frac{\ln K}{\eta_t}+\delta_t\right)
    =2\ln K\sum_t\frac{\delta_t}{\eta_t}+\Delta_T.
\end{split}\end{equation}
The only inequality in this equation replaces $\delta_t^2$ by
$\delta_t$, which is of no concern: the resulting $\Delta_T$ term adds
$2$ to the regret bound. We will now show
\begin{equation}\label{eq:delta_over_eta}
\frac{\delta_t}{\eta_t}\le\half v_t+\tfrac{1}{3}\delta_t.
\end{equation}
This supersedes the bound $\delta_t/\eta_t \le (e-2) v_t$ used by
\citet{CesaBianchiMansourStoltz2007}. Even though at first sight
circular, this form has two major advantages. Inclusion of the
overhead $\frac{1}{3}\delta_t$ will only affect smaller order terms of
the regret, but admits a significant reduction of the leading
constant. This gain directly percolates to our regret bounds
below. Additionally \eqref{eq:delta_over_eta} holds for all $\eta$,
which simplifies tuning considerably.

First note that~\eqref{eq:delta_over_eta} is clearly valid if
$\eta_t=\infty$. Assuming that $\eta_t$ is finite, we can obtain this
result by rewriting Bernstein's bound \eqref{eq:bernstein1} as
follows:
\[
\half v_t\ge\delta_t\cdot\frac{\eta_t}{2e^{\eta_t}-2\eta_t-2}=\frac{\delta_t}{\eta_t}-f(\eta_t)\delta_t,\quad\text{where}\quad
f(x)=\frac{e^x-\half x^2-x-1}{x e^x-x^2-x}.
\]
Remains to show that $f(x)\le1/3$ for all $x \geq 0$. After rearranging,
we find this to be the case if 
\[
(3-x)e^x \le\half x^2+2x+3.
\]
Taylor expansion of the left-hand side around zero reveals that
$(3-x)e^x=\half x^2+2x+3-\frac{1}{6}x^3 u e^u$ for some
$0\le u\le x$, from which the result follows.
The proof is completed by plugging~\eqref{eq:delta_over_eta} into~\eqref{eq:delta_square}.
\end{proof}

Combination of these results yields the following natural regret
bound, analogous to Theorem~5
of \citet{CesaBianchiMansourStoltz2007}.
\begin{theorem}\label{thm:variance_bound} For losses in the range $[0,1]$, AdaHedge's regret  is bounded by
  \[\Rah_T\le 2\sqrt{\Vah_T\ln K}+\tfrac{4}{3}\ln K+2.\]
\end{theorem}
\begin{proof}
  Lemma~\ref{lem:delta_square} is of the form
  \begin{equation}\label{eq:quadratic_inequality}
    (\Dah_T)^2 \leq a + b \Dah_T,
  \end{equation}
  with $a$ and $b$ nonnegative numbers. Solving for $\Dah_T$ then gives
  \[
    \Dah_T\le \half b+\half\sqrt{b^2+4a}\le \half b
    +\half(\sqrt{b^2}+\sqrt{4a})=\sqrt{a}+b,
  \]
  which by Lemma~\ref{lem:r_le_2delta} implies that
  \begin{equation}\label{eq:solved_quadratic}
    \Rah_T \leq 2\sqrt{a} + 2b.
  \end{equation}

  Plugging in the values $a = \Vah_T\ln K$ and $b = \frac{2}{3}\ln
  K+1$ from Lemma~\ref{lem:delta_square} completes the proof.
\end{proof}

This first regret bound for AdaHedge is difficult to interpret,
because the cumulative loss variance $\Vah_T$ depends on the actions
of the AdaHedge strategy itself (through the weights
$\w^\tn{ah}_t$). Below, we will derive a second regret bound for
AdaHedge that depends only on the data. However, AdaHedge has one
important property that is captured by this first result that is no
longer expressed by the worst-case bound we will derive below. Namely,
if the data are easy in the sense that there is a clear best expert,
say $k^*$, then the weights played by AdaHedge will concentrate on
that expert. If $w^\tn{ah}_{t,k^*}\to 1$ as $t$ increases, then the
loss variance must decrease: $v^\tn{ah}_t\to 0$. Thus,
Theorem~\ref{thm:variance_bound} suggests that the AdaHedge regret may
be bounded if the weights concentrate on the best expert sufficiently
quickly. This turns out to be the case: we can prove that the regret
is indeed bounded for the stochastic setting where the loss vectors
$\vloss_t$ are independent, and $E[L_{t,k^*}-L_{t,k}]=\Omega(t^\beta)$
for all $k\ne k^*$ and any $\beta>1/2$. This is an important feature
of AdaHedge when it is used as a stand-alone algorithm, and we provide
a proof for the previous version of the strategy in
\Citep{ErvenGrunwaldKoolenDeRooij2011}. See Section~\ref{sec:exp4} for
an example of concentration of the AdaHedge weights. We will not
pursue this further here because the Follow-the-Leader strategy also
incurs bounded loss in that case; we rather focus attention on how to
successfully compete with FTL in Section~\ref{sec:flipflop}.

We now proceed to derive a bound that depends only on the data, using
the same approach as the one taken
by~\citet{CesaBianchiMansourStoltz2007}. We first bound the
cumulative loss variance as follows:
\begin{lemma}\label{lem:worstcase}
  Suppose $\Hah_T \geq L_T^*$. Then, for losses in the range $[0,1]$,
  the cumulative loss variance for AdaHedge satisfies
  \[
  \Vah_T\le\frac{L_T^*(T-L_T^*)}{T}+2\Dah_T.
  \]
\end{lemma}
\begin{proof}
 The sum of variances is bounded by
  \[
  \Vah_T=\sum_t \vah_t\le\sum_t \hah_t(1-\hah_t)\le
  T\left(\frac{\Hah_T}{T}\right)\left(1-\frac{\Hah_T}{T}\right),
  \]
  where the first inequality is provided by Lemma~\ref{lem:bernstein},
  and the second is Jensen's.  Subsequently using $\Hah_T \geq L_T^*$
  (by assumption) and $\Hah_T\le L_T^*+2\Dah_T$ (by
  Lemma~\ref{lem:r_le_2delta}) yields
\[
\Vah_T\le\frac{(L_T^*+2\Dah_T)(T-L_T^*)}{T}\le\frac{L_T^*(T-L_T^*)}{T}+2\Dah_T,
\]
which was to be shown.
\end{proof}

This can be combined with Lemma~\ref{lem:delta_square}
and~\ref{lem:r_le_2delta} to obtain the following bound, which
improves the dominant term of Corollary~3 of
\citet{CesaBianchiMansourStoltz2007} by a factor of 2:
\begin{theorem}\label{thm:adahedge_regret} For losses in the range
  $[0,1]$, AdaHedge's regret is bounded by
\[\Rah_T\le 2\sqrt{\frac{L_T^*(T-L_T^*)}{T}\ln K}+\tfrac{16}{3}\ln K+2.\]
\end{theorem}
\begin{proof}
  If $\Hah_T< L_T^*$, then $\Rah_T<0$ and the result is clearly
  valid. But if $\Hah_T\ge L_T^*$, we can bound $\Vah_T$ using
  Lemma~\ref{lem:worstcase} and plug the result into
  Lemma~\ref{lem:delta_square} to get an inequality of the
  form~\eqref{eq:quadratic_inequality} with $a = L_T^*(T-L_T^*)/T\ln
  K$ and $b = \frac{8}{3}\ln K + 1$. Following the steps of the proof
  of Theorem~\ref{thm:variance_bound} with these modified values for
  $a$ and $b$ we arrive at the desired result.
\end{proof}

This is the best known bound for a Hedge algorithm where the regret is
expressed in terms of the loss rate $L^*_T/T$ of the best expert. Note
that the bound is maximized for $L^*_T=T/2$, in which case the
dominant term reduces to $\sqrt{T\ln K}$. This matches the best known
result of the same form \citep{Gerchinovitz2011}, and improves upon
the results of \citep{CesaBianchiLugosi2006} by a factor
$\sqrt{2}$. Alternatively, we can simplify our regret bound using
$(T-L_T^*)/T\le 1$ to obtain a dominant term of $2 \sqrt{L^*_T\ln
  K}$. This also improves the best known result
\citep{AuerCesaBianchiGentile2002} by a factor of $\sqrt{2}$. In both
cases, our analysis is more direct.

Note that the regret is small when the best expert either has a very
low loss rate, or a very high loss rate. The latter is important if
the algorithm is to be used for the scenario where we are provided
with a sequence of bounded gain vectors $\g_t$ rather than losses: we
can translate the gains into losses using $l_{t,k}=1-g_{t,k}$, and
then run AdaHedge. The bound expresses that we incur small regret even
if the best expert has a very small gain.

In the next section, we show how we can compete with FTL while
maintaining these excellent guarantees up to a constant factor.

\section{FlipFlop}\label{sec:flipflop}
AdaHedge balances the cumulative mixability gap $\Delta^\tn{ah}_T$ and
the mix loss regret $M^\tn{ah}_T - L_T^*$ by reducing $\eta^\tn{ah}_t$ as
necessary.  But, as we observed previously, if the data are not
hopelessly adversarial we might not need to worry about the mixability
gap: as Lemma~\ref{lem:bernstein} expresses, $\dah_t$ is also small
if the variance $v^\tn{ah}_t$ of the loss under the weights $w^\tn{ah}_{t,k}$ is
small, which is the case if the weight on the best expert $\max_k
w^\tn{ah}_{t,k}$ becomes close to one.

AdaHedge is able to exploit such a lucky scenario to an extent: as
explained in the discussion that follows
Theorem~\ref{thm:variance_bound},
if the weight of the best expert goes to one quickly, AdaHedge will
have a small cumulative mixability gap, and therefore, by
Lemma~\ref{lem:r_le_2delta}, a small regret. This happens, for
example, in the stochastic setting with independent, identically
distributed losses, when a single expert has the smallest expected loss. Similarly, in the experiment of Section~\ref{sec:exp4}, the
AdaHedge weights concentrate sufficiently quickly for the regret to be
bounded.

There is the potential for a nasty feedback loop, however. Suppose
there are a small number of difficult early trials, during which the
cumulative mixability gap increases relatively quickly. AdaHedge
responds by reducing the learning rate \eqref{eq:learning_rate}, with
the effect that the weights on the experts become more uniform. As a
consequence, the mixability gap in future trials may be larger than
what it would have been if the learning rate had stayed high, leading
to further unnecessary reductions of the learning rate, and so on. The
end result may be that AdaHedge behaves as if the data are difficult
and incurs substantial regret, even in cases where the regret of Hedge
with a fixed high learning rate, or of Follow-the-Leader, is bounded!
Precisely this phenomenon occurs in the experiment in
Section~\ref{sec:exp2} below: AdaHedge's regret is close to the
worst-case bound, whereas FTL hardly incurs any regret at all.

It appears, then, that we must \emph{either} hope that the data are
easy enough that we can make the weights concentrate quickly on a single
expert, by not reducing the learning rate at all; \emph{or} we fear the
worst and reduce the learning rate as much as we need to be able to
provide good guarantees. We cannot really interpolate between these two
extremes: an intermediate learning rate may not yield small regret in
favourable cases and may at the same time destroy any performance
guarantees in the worst case.

It is unclear a priori whether we can get away with keeping the
learning rate high, or that it is wiser to play it safe using AdaHedge.
The most extreme case of keeping the learning rate high, is the limit as
$\eta$ tends to $\infty$, for which Hedge reduces to Follow-the-Leader.
In this section we work out a strategy that combines the advantages of
FTL and AdaHedge: it retains AdaHedge's worst-case guarantees up to a
constant factor, but its regret is also bounded by a constant times the
regret of FTL (Theorem~\ref{thm:flipflop_regret}). Perhaps surprisingly,
this is not easy to achieve. To see why, imagine a scenario where the
average loss of the best expert is substantial (say, about $0.5$ per
round), whereas the \emph{regret} of either Follow-the-Leader or
AdaHedge, is small. Since our combination has to guarantee a similarly
small regret, it has only a very limited margin for error. We cannot,
for example, simply combine the two algorithms by recursively plugging
them into Hedge with a fixed learning rate, or into AdaHedge: the
performance guarantees we have for those methods of combination are too
weak. Even if both FTL and AdaHedge yield small regret on the original
problem, choosing the actions of FTL for some rounds and those of
AdaHedge for the other rounds may fail, because the regret is not
necessarily increasing, and we may end up picking each algorithm
precisely in those rounds where the other one is better.

These considerations motivate the FlipFlop strategy (superscript:\ ``ff'')
described in this section, where we carefully alternate between the
optimistic FTL strategy, and the worst-case-proof AdaHedge to get the
best of both worlds.

\subsection{Exploiting Easy Data by Following the Leader}

We first investigate the potential benefits of FTL over AdaHedge.
Lemma~\ref{lem:ftl_regret} below identifies the circumstances under
which FTL will perform well, which is when the number of leader changes
is small. It also shows that the regret for FTL is equal to the
cumulative mixability gap when FTL is interpreted as a Hedge strategy
with infinite learning rate.

\begin{lemma}\label{lem:ftl_regret}
  Let $c_t$ be an indicator for a leader change at time $t$: define
  $c_t = 1$ if $t=1$ or if there exists an expert $k$ such that $L_{t-1,k} =
  L^*_{t-1}$ while $L_{t,k} \neq L^*_t$, and $c_t = 0$ otherwise. Let
  $C_T = \sum_{t=1}^T c_t$ be the total number of leader changes up to
  time $T$. Then, for losses in the range $[0,1]$, the FTL regret satisfies
  \[\Rftl_T=\Delta_T^{\!(\infty)}\leq C_T.\]
\end{lemma}

\begin{proof}
  We have $M_T^{(\infty)} = L^*_T$ by \bpref{it:bayesbound}, and
  consequently $\Rftl_T = \Delta_T^{(\infty)} + M_T^{(\infty)} - L^*_T =
  \Delta_T^{(\infty)}$.
  
  To bound $\Delta_T^{(\infty)}$, notice that, for any $t$ such that
  $c_t = 0$, all leaders remained leaders and incurred identical loss. It follows that $m_t^{(\infty)} = L^*_t -
  L^*_{t-1} = h_t^{(\infty)}$ and hence
  $\delta_t^{(\infty)} = 0$. By bounding $\delta_t^{(\infty)} \leq 1$
  for all other $t$ we obtain
  \begin{equation*}
    \Delta_T^{(\infty)}
      = \sum_{t=1}^T \delta_t^{(\infty)}
      = \sum_{t \colon c_t = 1} \delta_t^{(\infty)}
      \leq \sum_{t \colon c_t = 1} 1
      = C_T,
  \end{equation*}
  as required.
\end{proof}

We see that the regret for FTL is bounded by the number of leader
changes. This is a natural measure of the difficulty of the problem,
because it remains small whenever a single expert makes the best
predictions on average, even in the scenario described above, in which
AdaHedge gets caught in a feedback loop. One easy example where FTL
outperforms AdaHedge is when the losses are $(1,0)$, $(1,0)$, $(0,1)$,
$(1,0)$, \dots Then the FTL regret is at most one, whereas AdaHedge's
performance is close to the worst case bound. This scenario is
discussed further in the experiments, Section~\ref{sec:exp2}.

\subsection{FlipFlop}
In the following analysis we will assume, as before, that the losses
satisfy $\vloss_t\in[0,1]^K$; see Section~\ref{sec:invariance} for
discussion of the general case. FlipFlop is a Hedge strategy in the
sense that it uses exponential weights defined by
\eqref{eq:adahedge_defs}, but the learning rate $\eta_t^\tn{ff}$ now
alternates between infinity, such that the algorithm behaves like FTL,
and the AdaHedge value, which decreases as a function of the
mixability gap accumulated over the rounds where AdaHedge is used. In
Definition~\ref{def:regimes} below, we will specify the ``flip''
regime $\F_t$, which is the subset of times $\{1,\ldots,t\}$ where we
follow the leader by using an infinite learning rate, and the ``flop''
regime $\A_t =\{1,\ldots,t\}\setminus \F_t$, which is the set of times
where the learning rate is determined by AdaHedge (mnemonic: the
position of the bar refers to the value of the learning rate). We
accumulate the mixability gap, the mix loss and the variance for these
two regimes separately:
\begin{align*}
\Dl_T&=\sum_{t\in \F_T}\delta^\tn{ff}_t;&\Bl_T&=\sum_{t\in \F_T}m^\tn{ff}_t;
&&&\text{(flip)}\\
\Dw_T&=\sum_{t\in \A_T}\delta^\tn{ff}_t;&\Bw_T&=\sum_{t\in
  \A_T}m^\tn{ff}_t;&\Vw_T&=\sum_{t\in \A_T}v^\tn{ff}_t.&\text{(flop)}
\end{align*}
We also change the learning rate from its definition for AdaHedge
in~\eqref{eq:learning_rate} to the following, which differentiates
between the two regimes of the strategy:
\begin{equation}\label{eq:ff_eta}
\eta^\tn{ff}_t=\begin{cases}\eta^\tn{flip}_t&\text{if $t\in \F_t$,}\\\eta^\tn{flop}_t&\text{if $t\in \A_t$,}\end{cases}
\quad \text{where} \quad
\eta^\tn{flip}_t=\eta^\tn{ftl}_t=\infty
\quad \text{and} \quad
\eta^\tn{flop}_t=\frac{\ln K}{\Dw_{t-1}}.
\end{equation}
Note that while the learning rates are defined separately for the two
regimes, the exponential weights \eqref{eq:adahedge_defs} of the experts are still always
determined using the cumulative losses $L_{t,k}$ over \emph{all} rounds.
We also point out that, for rounds $t\in\A_T$, the learning rate
$\eta^\tn{ff}_t=\eta^{\tn{flop}}_t$ is not equal to $\eta^\tn{ah}_t$,
because it uses $\Dw_{t-1}$ instead of $\Dah_{t-1}$. For this reason,
the FlipFlop regret may be either better or worse than the
AdaHedge regret; our results below only preserve the regret
\emph{bound} up to a constant factor. In contrast, we do compete with
the \emph{actual} regret of FTL.

It remains to define the ``flip'' regime $\F_t$ and the ``flop''
regime $\A_t$, which we will do by specifying the times at which to
switch from one to the other. FlipFlop starts optimistically, with an
epoch of the ``flip'' regime, which means it follows the leader, until
$\Dl_t$ becomes too large compared to $\Dw_t$. At that point it
switches to an epoch of the ``flop'' regime, and keeps using
$\eta^\tn{flop}_t$ until $\Dw_t$ becomes too large compared to
$\Dl_t$. Then the process repeats with the next epochs of the ``flip''
and ``flop'' regimes. The regimes are determined as follows:
\begin{definition}[FlipFlop's Regimes]\label{def:regimes}
  Let $\phi > 1$ and $\alpha > 0$ be parameters of the
  algorithm. Then
  \begin{itemize}
  \item FlipFlop starts in the ``flip'' regime.
  \item If $t$ is the earliest time since the start of a ``flip''
    epoch where $\Dl_t>(\phi/\alpha)\Dw_t$, then the transition to the
    subsequent ``flop'' epoch occurs between rounds $t$ and $t+1$. (Recall
    that during ``flip'' epochs $\Dl_t$ increases in $t$ whereas
    $\Dw_t$ is constant.) 
  \item Vice versa, if $t$ is the earliest time since the start of a
    ``flop'' epoch where $\Dw_t>\alpha \Dl_t$, then the transition to
    the subsequent ``flip'' epoch occurs between rounds $t$ and $t+1$.
  \end{itemize}
\end{definition}
This completes the definition of the FlipFlop strategy. See
Figure~\ref{fig:flipflop} for a {\sc matlab} implementation.

\begin{figure}
\centering
\begin{minipage}{\linewidth}
\footnotesize
\begin{alltt}
\textsl{\% Returns the losses of FlipFlop}
\textsl{\% l(t,k) is the loss of expert k at time t; phi > 1 and alpha > 0 are parameters}
function h = \textbf{flipflop}(l\textbf{, alpha, phi})
    [T, K]  = size(l);
    h       = nan(T,1);
    L       = zeros(1,K);
    \textbf{Delta  = [0 0];}
    \textbf{scale  = [phi/alpha alpha];}
    \textbf{regime = 1;} \textsl{\% 1=FTL, 2=AH}

    for t = 1:T
        \textbf{if regime==1, eta = Inf; else eta = log(K)/Delta(2); end}
        [w, Mprev] = mix(eta, L);
        h(t) = w * l(t,:)';
        L = L + l(t,:);
        [~, M] = mix(eta, L);
        delta = max(0, h(t)-(M-Mprev));
        \textbf{Delta(regime) = Delta(regime) + delta;}
        \textbf{if Delta(regime) > scale(regime) * Delta(3-regime)}
          \textbf{regime = 3-regime;}
        \textbf{end}
    end
end
\end{alltt}
\end{minipage}
\caption{FlipFlop, with new ingredients in boldface\label{fig:flipflop}}
\end{figure}


The analysis proceeds much like the analysis for AdaHedge. We first
show that, analogously to Lemma~\ref{lem:r_le_2delta}, the FlipFlop
regret can be bounded in terms of the cumulative mixability gap; in
fact, we can use the \emph{smallest} cumulative mixability gap that we
encountered in either of the two regimes, at the cost of slightly
increased constant factors.  This is the fundamental building block in
our FlipFlop analysis. We then proceed to develop analogues of
Lemmas~\ref{lem:delta_square} and~\ref{lem:worstcase}, whose proofs do
not have to be changed much to apply to FlipFlop. Finally, all these
results are combined to bound the regret of FlipFlop in
Theorem~\ref{thm:flipflop_regret}, which is the main result of this
paper.

\begin{lemma}[FlipFlop version of Lemma~\ref{lem:r_le_2delta}]\label{lem:flipflop_regret}
Suppose the losses take values in $[0,1]$. Then the following two bounds
hold simultaneously for the regret of the FlipFlop strategy with
parameters $\phi > 1$ and $\alpha > 0$:
\begin{align}
\Rff_T&\le\left(\frac{\phi\alpha}{\phi-1}+2\alpha+1\right)\Dl_T
+\frac{\alpha\phi}{\phi-1}
+2\alpha
;\label{eq:flipflop_regret_lucky}\\
\Rff_T&\le\left(\frac{\phi}{\phi-1}+\frac{\phi}{\alpha}+2\right)\Dw_T+\frac{\phi}{\alpha}.\label{eq:flipflop_regret_worstcase}
\end{align}
\end{lemma}
\begin{proof}
The regret can be decomposed as
\begin{equation}\label{eq:regret_flipflop}
  \Rff_T=\Hff_T-L_T^*=\Dl_T+\Dw_T+\Bl_T+\Bw_T-L_T^*.
\end{equation}
%
%
%
%
Our first step will be to bound the mix loss $\Bl_T+\Bw_T$ in terms of the
mix loss $M_T^\tn{flop}$ of the auxiliary strategy
that uses $\eta_t^{\tn{flop}}$ for all $t$. As $\eta^\tn{flop}_t$ is
nonincreasing, we can then apply Lemma~\ref{lem:bayesbound} and
\bpref{it:bayesbound} to further bound
\begin{equation}\label{eq:drop_bayes_down}
M_T^\tn{flop}\le M_T^{(\eta_T^{\tn{flop}})}\le L_T^*+\frac{\ln K}{\eta^\tn{flop}_T}
= L_T^*+\Dw_{T-1} \le L_T^*+\Dw_T.
\end{equation}
Let $0=u_1<u_2<\ldots<u_b<T$ denote the times just before
the epochs of the ``flip'' regime begin, i.e.\ round $u_i+1$ is the
first round in the $i$-th ``flip'' epoch. Similarly let
$0<v_1<\ldots<v_b\leq T$ denote the times just before the epochs of
the ``flop'' regime begin, where we artificially define $v_b = T$ if
the algorithm is in the ``flip'' regime after $T$ rounds. These
definitions ensure that we always have $u_b < v_b \leq T$. For the mix
loss in the ``flop'' regime we have
\begin{equation}\label{eq:flopmixloss}
\Bw_T =
(M^\tn{flop}_{u_2}-M^\tn{flop}_{v_1})+(M^\tn{flop}_{u_3}-M^\tn{flop}_{v_2})
+\ldots+(M^\tn{flop}_{u_b}-M^\tn{flop}_{v_{b-1}})+(M^\tn{flop}_T-M^\tn{flop}_{v_b}).
\end{equation}
Let us temporarily write $\eta_t = \eta_t^\tn{flop}$ to avoid double
superscripts. For the ``flip'' regime, the properties in
Lemma~\ref{lem:basic_properties}, together with the observation that
$\eta_t^{\tn{flop}}$ does not change during the ``flip'' regime, give
\begin{align}
\Bl_T &= \sum_{i=1}^b \left(M^{(\infty)}_{v_i}-M^{(\infty)}_{u_i}\right)
    = \sum_{i=1}^b \left(M^{(\infty)}_{v_i}-L^*_{u_i}\right)
    \le \sum_{i=1}^b
    \left(M^{(\eta_{v_i})}_{v_i}-L^*_{u_i}\right)\notag\\
    &\le \sum_{i=1}^b
    \left(M^{(\eta_{v_i})}_{v_i}-M^{(\eta_{v_i})}_{u_i}
    + \frac{\ln K}{\eta_{v_i}}\right)
    = \sum_{i=1}^b
    \left(M^{\tn{flop}}_{v_i}-M^{\tn{flop}}_{u_i}
    + \frac{\ln K}{\eta_{u_i+1}}\right)\notag\\
&=\left(M^\tn{flop}_{v_1}-M^\tn{flop}_{u_1}\right)+\left(M^\tn{flop}_{v_2}-M^\tn{flop}_{u_2}\right)+\ldots+\left(M^\tn{flop}_{v_b}-M^\tn{flop}_{u_b}\right)+\sum_{i=1}^b\Dw_{u_i}.\label{eq:flipmixloss}
\end{align}
From the definition of the  regime changes
(Definition~\ref{def:regimes}), we know the value of $\Dw_{u_i}$ very
accurately at the time $u_i$ of a change from a ``flop'' to a
``flip'' regime: 
\[
\Dw_{u_i} > \alpha \Dl_{u_i} = \alpha \Dl_{v_{i-1}} > \phi \Dw_{v_{i-1}} = \phi \Dw_{u_{i-1}}
.
\]
By unrolling from low to high $i$, we see that 
\[
\sum_{i=1}^b \Dw_{u_i} \le 
\sum_{i=1}^b \phi^{1-i} \Dw_{u_b}  \le 
\sum_{i=1}^\infty \phi^{1-i} \Dw_{u_b} =
\frac{\phi}{\phi-1} \Dw_{u_b}
.
\]
Adding up \eqref{eq:flopmixloss} and \eqref{eq:flipmixloss}, we
therefore find that the total mix loss is bounded by
\begin{equation}\label{eq:bayes_loss}
\Bl_T+\Bw_T
\leq M_T^{\tn{flop}}+\sum_{i=1}^b\Dw_{u_i}
\le M_T^{\tn{flop}}+\frac{\phi}{\phi-1}\Dw_{u_b}
\le L_T^*+\left(\frac{\phi}{\phi-1}+1\right)\Dw_T
\end{equation}
where the last inequality uses~\eqref{eq:drop_bayes_down}. Combination
with~\eqref{eq:regret_flipflop} yields
\begin{equation}\label{eq:biglemmaregret}
\Rff_T\le\left(\frac{\phi}{\phi-1}+2\right)\Dw_T+\Dl_T.
\end{equation}

Our next goal is to relate $\Dw_T$ and $\Dl_T$: by construction of the
regimes, they are always within a constant factor of each
other. First, suppose that after $T$ trials we are in the $b$th epoch
of the ``flip'' regime, that is, we will behave like FTL in round
$T+1$.  In this state, we know from Definition~\ref{def:regimes} that
$\Dw_T$ is stuck at the value that prompted the start of the current
epoch; this pinpoints its value up to one. At the same time, we know
that $\Dl_T$ is large enough to have prompted the start of the
$(b-1)$st flop epoch, but not large enough to trigger the next regime
change. From this we can deduce the following bounds:
\begin{align*}
(\Dw_T -1)/\alpha &\le \Dl_T \le \frac{\phi}{\alpha} \Dw_T
\intertext{%
On the other hand, if after $T$ rounds we are in the $b$th epoch of
the ``flop'' regime, then a similar reasoning yields
}
\frac{\alpha}{\phi}(\Dl_T-1) & \le \Dw_T \le \alpha \Dl_T
\end{align*}
In both cases, it follows that
\begin{align*}
\Dw_T&<\alpha\Dl_T+\alpha;\\
\Dl_T&<\displaystyle\frac{\phi}{\alpha}\Dw_T+\frac{\phi}{\alpha}.
\end{align*}
The two bounds of the lemma are obtained by plugging first one, then
the other of these bounds into~\eqref{eq:biglemmaregret}.
\end{proof}

\begin{lemma}[FlipFlop version of Lemma~\ref{lem:delta_square}]\label{lem:flipflop_delta}
  Suppose the losses take values in $[0,1]$. Then the cumulative
  mixability gap for the ``flop'' regime is bounded by the cumulative
  variance of the losses for the ``flop'' regime:
  \[
  \left(\Dw_T\right)^2\le\Vw_T\ln K+(1+\tfrac{2}{3}\ln K)\Dw_T.
  \]
\end{lemma}
\begin{proof}
  The proof is analogous to the proof of Lemma~\ref{lem:delta_square},
  with $\Dw_T$ instead of $\Dah_T$, $\Vw_T$ instead of $\Vah_T$, and
  using $\eta_t=\eta_t^\tn{flop}=\ln(K)/\Dw_{t-1}$ instead of
  $\eta_t=\eta_t^\tn{ah}=\ln(K)/\Delta^\tn{ah}_{t-1}$.  Furthermore,
  we only need to sum over the rounds $\A_T$ in the ``flop'' regime, because
  $\Dw_T$ does not change during the ``flip'' regime.
\end{proof}

We could use this result to prove an analogue of
Theorem~\ref{thm:variance_bound} for FlipFlop, but this would be
tedious; we therefore proceed directly to bound the variance in terms of
the loss rate of the best expert. The following Lemma provides the
equivalent of Lemma~\ref{lem:worstcase} for FlipFlop. It can probably be
strengthened to improve the lower order terms; we provide the version
that is easiest to prove.

\begin{lemma}[FlipFlop version of Lemma~\ref{lem:worstcase}]\label{lem:ff_worstcase}
  Suppose $\Hff_T \geq L_T^*$. Then, for losses in the range $[0,1]$,
  the cumulative loss variance for FlipFlop with parameters $\phi > 1$
  and $\alpha > 0$ satisfies
  \[\Vw_T\le\frac{L_T^*(T-L_T^*)}{T}+\left(\frac{\phi}{\phi-1}+\frac{\phi}{\alpha}+2\right)\Dw_T+\frac{\phi}{\alpha}.\]
\end{lemma}
\begin{proof}
  The sum of variances satisfies
  \[
  \Vw_T=\sum_{t\in \A_T} v^\tn{ff}_t\le\sum_{t=1}^T v^\tn{ff}_t\le \sum_{t=1}^T h^\tn{ff}_t(1-h^\tn{ff}_t)\le T\left(\frac{\Hff_T}{T}\right)\left(1-\frac{\Hff_T}{T}\right),
  \]
  where the first inequality simply adds the variances for FTL rounds
  (which are often all zero), the second is
  Lemma~\ref{lem:bernstein}, and the third is Jensen's
  inequality. Subsequently using $L_T^*\le\Hff_T$ (by assumption) and,
  from Lemma~\ref{lem:flipflop_regret}, $\Hff_T\le L_T^*+c$, where $c$
  denotes the right hand side of the
  bound~\eqref{eq:flipflop_regret_worstcase}, we find
\[
\Vw_T\le\frac{(L_T^*+c)(T-L_T^*)}{T}\le\frac{L_T^*(T-L_T^*)}{T}+c,
\]
which was to be shown.
\end{proof}

Combining Lemmas~\ref{lem:flipflop_regret}, \ref{lem:flipflop_delta}
and \ref{lem:ff_worstcase}, we obtain our main result:

\begin{theorem}[Main Theorem, FlipFlop version of
  Theorem~\ref{thm:adahedge_regret}]\label{thm:flipflop_regret} Suppose
  the losses take values in $[0,1]$. 
Then the
  regret for FlipFlop
  with doubling parameters $\phi > 1$ and $\alpha > 0$
  simultaneously satisfies the bounds
  \begin{align*}
    \Rff_T&\le \left(\frac{\phi\alpha}{\phi-1}+2\alpha+1\right)\Rftl_T
    +\frac{\alpha\phi}{\phi-1}
+2\alpha
,\\
    \Rff_T&\le c_1\sqrt{\frac{L_T^*(T-L_T^*)}{T}\ln K}+
    c_1(c_1+\tfrac{2}{3})\ln K + c_1\sqrt{c_2 \ln K} + c_1 + c_2,
  \end{align*}
  where $\displaystyle c_1=\frac{\phi}{\phi-1}+\frac{\phi}{\alpha}+2$
  and $\displaystyle
  c_2=\frac{\phi}{\alpha}$.
\end{theorem}

This shows that, up to a multiplicative factor in the regret, FlipFlop
is always as good as the best of Follow-the-Leader and AdaHedge's
bound. Of course, if AdaHedge significantly outperforms its bound, it
is not guaranteed that FlipFlop will outperform the bound in the same
way.

In the experiments in Section~\ref{sec:experiments} we demonstrate that
the multiplicative factor is not just an artifact of the bounds, but can
actually be observed on simulated data. 

\begin{proof}
  From Lemma~\ref{lem:ftl_regret}, we know that
  $\Dl_T\le\Delta_T^{(\infty)}=\Rftl_T$. Substitution in~\eqref{eq:flipflop_regret_lucky} of
  Lemma~\ref{lem:flipflop_regret} yields the first inequality.  

  For the second inequality, note that $L_T^*>\Hff_T$ means the regret is
  negative, in which case the result is clearly valid. We may
  therefore assume w.l.o.g.\ that $L_T^*\le\Hff_T$ and
  apply Lemma~\ref{lem:ff_worstcase}. Combination with
  Lemma~\ref{lem:flipflop_delta} yields
  \[
  \left(\Dw_T\right)^2\le\Vw_T\ln K+(1+\tfrac{2}{3}\ln K)\Dw_T\le \frac{L_T^*(T-L_T^*)}{T}\ln K+c_2\ln K+c_3\Dw_T,
  \]
  where $c_3=1+(c_1+\tfrac{2}{3})\ln K$. We now solve this quadratic
  inequality as in \eqref{eq:quadratic_inequality} and relax it using
  $\sqrt{a+b}\le\sqrt{a}+\sqrt{b}$ for nonnegative numbers
  $a,b$ to obtain
  \begin{align*}
  \Dw_T
  &\le
  \sqrt{\frac{L_T^*(T-L_T^*)}{T}\ln K + c_2\ln K}+ c_3\\
  &\le 
  \sqrt{\frac{L_T^*(T-L_T^*)}{T}\ln K}+(c_1+\tfrac{2}{3})\ln K +
  \sqrt{c_2\ln K}+1.
  \end{align*}
  In combination with Lemma~\ref{lem:flipflop_regret}, this yields the
  second bound of the theorem.
\end{proof}

Finally, we propose to select the parameter values that minimize the constant factor in front of the leading terms of these regret bounds.

\begin{corollary}
  The parameter values $\phi^*=2.37$ and $\alpha^*=1.243$ approximately
  minimize the worst of the two leading factors in the bounds of
  Theorem~\ref{thm:flipflop_regret}. The regret for FlipFlop with these parameters
  is simultaneously bounded by
  \begin{align*}
    \Rff_T&\le 5.64\Rftl_T+4.64,\\
    \Rff_T&\le 5.64\sqrt{\frac{L_T^*(T-L_T^*)}{T}\ln K}+ 35.53 \ln K + 7.78
    \sqrt{\ln K} + 7.54.
  \end{align*}
\end{corollary}

\begin{proof}
  The leading factors $f(\phi,\alpha) =
  \frac{\phi\alpha}{\phi-1}+2\alpha+1$ and $g(\phi,\alpha) =
  \frac{\phi}{\phi-1}+\frac{\phi}{\alpha}+2$ are respectively
  increasing and decreasing in $\alpha$. They are equalized for
  $\alpha(\phi) = \big(2\phi-1 +
  \sqrt{12\phi^3-16\phi^2+4\phi+1}\big)/(6\phi-4)$. The analytic
  solution for the minimum of $f(\phi,\alpha(\phi))$ in $\phi$ is too
  long to reproduce here, but it is approximately equal to $\phi^* =
  2.37$, at which point $\alpha(\phi^*) \approx 1.243$.
\end{proof}

\section{Invariance to Rescaling and Translation}\label{sec:invariance}

In the previous two sections, we have assumed, for simplicity, that the
losses $\ell_{t,k}$ were translated and normalised to take values in the
interval $[0,1]$. Although this is a common assumption in the
literature, it requires \emph{a priori} knowledge of the range of the
losses. One would therefore prefer algorithms that do not require the
losses to be normalised. As discussed by
\citet{CesaBianchiMansourStoltz2007}, the regret bounds for such
algorithms should not change when losses are translated (because this
does not change the regret) and should scale by $\sigma$ when the losses
are scaled by a factor $\sigma > 0$ (because the regret scales by
$\sigma$). They call such regret bounds \emph{fundamental} and show that
most of the methods they introduce satisfy such fundamental bounds.

Here we go even further: it is not just our bounds that are fundamental,
but also our algorithms, which do not change their output weights if
the losses are scaled or translated.
\begin{theorem}\label{thm:fundamental}
  Both AdaHedge and FlipFlop are invariant to translation and
  rescaling of the losses. Starting with losses
  $\vloss_1,\ldots,\vloss_T$, obtain rescaled, translated losses
  $\vloss'_1,\ldots,\vloss'_T$ by picking any $\sigma>0$ and arbitrary
  reals $\tau_1,\ldots,\tau_T$, and setting
  $\ell'_{t,k}=\sigma\ell_{t,k}+\tau_t$ for $t=1,\ldots,T$ and
  $k=1,\ldots,K$. Both AdaHedge and FlipFlop issue the exact same
  sequence of weights $\w'_t=\w_t$ on $\vloss'_t$ as they do on $\vloss_t$.
\end{theorem}
\begin{proof}
  We annotate any quantity with a prime to denote that it is defined
  with respect to the data set $\vloss'_t$. We omit the algorithm name
  from the superscript. First consider AdaHedge. We will prove the
  following relations by induction on $t$:
  \begin{equation}\label{eq:ind_hyp}
  \Delta'_{t-1}=\sigma\Delta_{t-1};\qquad
  \eta'_t=\frac{\eta_t}{\sigma};\qquad \w'_t=\w_t.
  \end{equation}
  For $t=1$, these are valid since $\Delta'_0=\sigma\Delta_0=0$,
  $\eta'_1=\eta_1/\sigma=\infty$, and $\w'_1=\w_1$ are uniform. Now assume
  towards induction that \eqref{eq:ind_hyp} is valid for some
  $t\in\{1,\ldots,T\}$. We can then compute the following values from
  their definition: $h'_t=\w'_t\dot\ell'_t=\sigma h_t+\tau_t$;
  $m'_t=-(1/\eta'_t)\ln(\w'_t\dot e^{-\eta'_t\vloss'_t})=\sigma
  m_t+\tau_t$; $\delta'_t=h'_t-m'_t=\sigma(h_t-m_t)=\sigma\delta_t$.
  Thus, the mixability gaps are also related by the scale factor
  $\sigma$. From there we can reestablish the induction hypothesis for
  the next round: we have
  $\Delta'_t=\Delta'_{t-1}+\delta'_t=\sigma\Delta_{t-1}+\sigma\delta_t=\sigma\Delta_t$,
  and $\eta'_{t+1}=\ln(K)/\Delta'_t=\eta_{t+1}/\sigma$. 
  For the weights we get $\w'_{t+1}\propto
  e^{-\eta'_{t+1}\dot\L'_t}\propto
  e^{-(\eta_t/\sigma)\dot(\sigma\L_t)}\propto \w_{t+1}$, which means
  the two must be equal since both sum to one. Thus the relations
  of~\eqref{eq:ind_hyp} are also valid for time $t+1$, proving the
  result for AdaHedge.

  For FlipFlop, if we assume regime changes occur at the same times
  for $\vloss'$ and $\vloss$, then similar reasoning reveals
  ${\Dl_t}'=\sigma\Dl_t$; $\Dw'_t=\sigma\Dw_t$,
  ${\eta'}^\tn{flip}_t=\eta^\tn{flip}_t/\sigma=\infty$,
  ${\eta'}^\tn{flop}_t=\eta^\tn{flop}_t/\sigma$, and
  $\w'_t=\w_t$. Remains to check that the regime changes do indeed
  occur at the same times. Note that in Definition~\ref{def:regimes},
  the ``flop'' regime is started when ${\Dl_t}'>(\phi/\alpha)\Dw'_t$,
  which is equivalent to testing $\Dl_t>(\phi/\alpha)\Dw_t$ since both
  sides of the inequality are scaled by $\sigma$. Similarly, the
  ``flip'' regime starts when $\Dw'_t>\alpha{\Dl_t}'$, which is
  equivalent to the test $\Dw_t>\alpha\Dl_t$.
\end{proof}


Making our bounds fundamental is a simple corollary of
Theorem~\ref{thm:fundamental}. For AdaHedge the result is a slight
improvement of the bound \eqref{eq:stoltz} for the CBMS algorithm by
\citet{CesaBianchiMansourStoltz2007}.
\begin{corollary}
  Fix arbitrary losses $\vloss_1, \ldots, \vloss_T$ in $\mathbb R$,
  and let
  \begin{align*}
    \mu_t &= \min_k \ell_{t,k} &
    \sigma &= \max_{t \in \{1,\ldots,T\}} \max_k (\ell_{t,k} - \mu_t)
  \end{align*}
  be the minimal loss in round $t$ and the scale of the losses,
  respectively. Then, without modification, AdaHedge and FlipFlop
  satisfy the regret bounds
  \begin{align*}
    \Rah_T&\le 2\sqrt{\frac{N_T^*(\sigma T-N_T^*)}{T}\ln K}+ 
    \sigma \left(
      \tfrac{16}{3}
      \ln K+2
    \right),\\
  \intertext{and}
    \Rff_T&\le \left(\frac{\phi\alpha}{\phi-1}+2\alpha+1\right)\Rftl_T
    +\sigma \left(
      \frac{\alpha\phi}{\phi-1}
      +2\alpha 
    \right)
,\\
    \Rff_T&\le c_1\sqrt{\frac{N_T^*(\sigma T-N_T^*)}{T}\ln K}+
    \sigma \left(
      c_1(c_1+\tfrac{2}{3}) \ln K + c_1 \sqrt{c_2 \ln K} +
      c_1 + c_2
  \right),
  \end{align*}
  where $N_T^* = L_T^* - \sum_{t=1}^T \mu_t$ is the optimally translated
  loss of the best expert, and $c_1$ and $c_2$ are the same constants as
  in Theorem~\ref{thm:flipflop_regret}.
\end{corollary}

\begin{proof}
  Define the normalised losses $\ell'_{t,k}=(\ell_{t,k}-\mu_t)/\sigma$,
  and let $\Rahprime$, $\Rffprime$ and $\Rftlprime$ respectively denote
  the regret of AdaHedge, FlipFlop and Follow-the-Leader when run on
  these losses. Also let $L_T^{'*} = (L_T^* - \sum_{t=1}^T
  \mu_t)/\sigma$ denote the corresponding loss of the best expert. Then
  we have $N_T^* = \sigma L_T^{'*}$ and
  by Theorem~\ref{thm:fundamental} also $\Rah = \sigma \Rahprime$, $\Rff
  = \sigma \Rffprime$ and $\Rftl = \sigma\Rftlprime$. The corollary
  follows by plugging these identities into the bounds obtained by
  applying Theorems~\ref{thm:adahedge_regret} and
  \ref{thm:flipflop_regret} to the normalised losses
  $\vloss'_1,\ldots,\vloss'_T$.
\end{proof}

\section{Experiments}\label{sec:experiments}
We performed four experiments on artificial data, designed to clarify
how the learning rate determines performance in a variety of Hedge
algorithms. We have kept the experiments as simple as possible: the
data are deterministic, and involve two experts. In each case, the
data consist of one initial hand-crafted loss vector, followed by a
sequence of 999 loss vectors which are either $(0~1)$ or
$(1~0)$. The data are generated by sequentially appending the loss
vector that brings the cumulative loss difference $L_{t,1}-L_{t,2}$
closer to a target $f_\xi(t)$, where $\xi\in\{1,2,3,4\}$ indexes a
particular experiment. Each $f_\xi:[0,\infty)\to[0,\infty)$ is a
nondecreasing function with $f_\xi(0)=0$; intuitively, it expresses
how much better expert~2 is than expert~1 as a function of time. The
functions $f_\xi$ change slowly enough that our construction has the
property $|L_{t,1}-L_{t,2}-f_\xi(t)|\le 1$ for all $t$.

For each experiment, we first plot $\regret_T^{(\eta)}$, the regret of
the Hedge algorithm as a function of the fixed learning rate
$\eta$. We subsequently plot the regret $\regret^\tn{alg}_t$ as a
function of the time $t=1,\ldots,T=1000$, for each of the following
algorithms ``alg'':

\begin{enumerate}
\item Follow-the-Leader (Hedge with learning rate $\infty$),
\item Hedge with fixed learning rate $\eta=1$,
\item Hedge with the learning rate that optimizes the worst-case
  bound~\eqref{eq:wcbound} ($\eta=\sqrt{8\ln(K)/T}\approx0.0745$); we
  will call this algorithm ``safe Hedge'' for brevity,\label{it:safe}
\item AdaHedge,
\item FlipFlop,
\item Hazan and Kale's \citeyear{HazanKale2008} 
algorithm,
  using the fixed learning rate that optimises the bound provided in
  their paper.
\item NormalHedge, described by~\cite{ChaudhuriFreundHsu2009}.
\end{enumerate}
Note that the safe Hedge strategy (the third item above) can only be
used in practice if the horizon $T$ is known in advance. Hazan and
Kale's algorithm (the sixth item) additionally requires precognition
of the losses incurred by the various actions up until $T$. In
practice these algorithms would have to be used in conjunction with
the doubling trick, which would result in substantially worse, and
harder to interpret, results.

We include algorithms~6 and~7 because, as we explained in
Section~\ref{sec:related}, they are the state of the art in
Hedge-style algorithms.  To reduce clutter, we omit results for the
algorithm described in~\citet{CesaBianchiMansourStoltz2007}; its
behaviour is very similar to that of AdaHedge. Below we provide an
exact description of each experiment, and discuss the results.

\subsection{Experiment~1. Worst case for FTL}\label{sec:exp1}
The experiment is defined by $\vloss_1=(\half~0)$, and
$f_1(t)=0$. This yields a loss matrix $\vloss$ that starts as follows:
\[\begin{pmatrix}
1/2&0&1&0&1&\cdots\\
0&1&0&1&0&\cdots
\end{pmatrix}^\top.
\]
These data are the worst case for FTL: each round, the leader incurs
loss one, while each of the two individual experts only receives a
loss once every two rounds. Thus, the FTL regret increases by one
every two rounds and ends up around 500. For any learning rate $\eta$,
the weights used by the Hedge algorithm are repeated every two rounds,
so the regret $H_t-L^*_t$ increases by the same amount every two rounds:
the regret increases linearly in $t$ for every fixed $\eta$ that does
not vary with $t$. However, the constant of proportionality can be
reduced greatly by reducing the value of $\eta$, as the top graph in
Figure~\ref{fig:exp1} shows: for $T=1000$, the regret becomes negligible
for any $\eta$ less than about $0.01$. Thus, in this experiment, a
learning algorithm must reduce the learning rate to shield itself from
incurring an excessive overhead.

The bottom graph in Figure~\ref{fig:exp1} shows the expected breakdown
of the FTL algorithm; Hedge with fixed learning rate $\eta=1$ also
performs quite badly. When $\eta$ is reduced to the value that
optimises the worst-case bound, the regret becomes competitive with
that of the other algorithms. Note that Hazan and Kale's algorithm has
the best performance; this is because its learning rate is tuned in
relation to the bound proved in the paper, which has a relatively
large constant in front of the leading term. As a consequence the
algorithm always uses a relatively small learning rate, which turns
out to be helpful in this case but harmful in later experiments.

The FlipFlop algorithm behaves as theory suggests it should: its
regret increases alternately like the regret of AdaHedge and the
regret of FTL. The latter performs horribly, so during those intervals
the regret increases quickly, on the other hand the FTL intervals are
relatively short-lived so they do not harm the regret by more than a
constant factor.

The NormalHedge algorithm still has acceptable performance, although
it is relatively large in this experiment; we have no explanation for
this but in fairness we do observe good performance of NormalHedge in
the other three experiments as well as in numerous further unreported
simulations.

\begin{figure}[!ht]
\centering
\mbox{\kern-1cm\includegraphics[height=0.95\textheight]{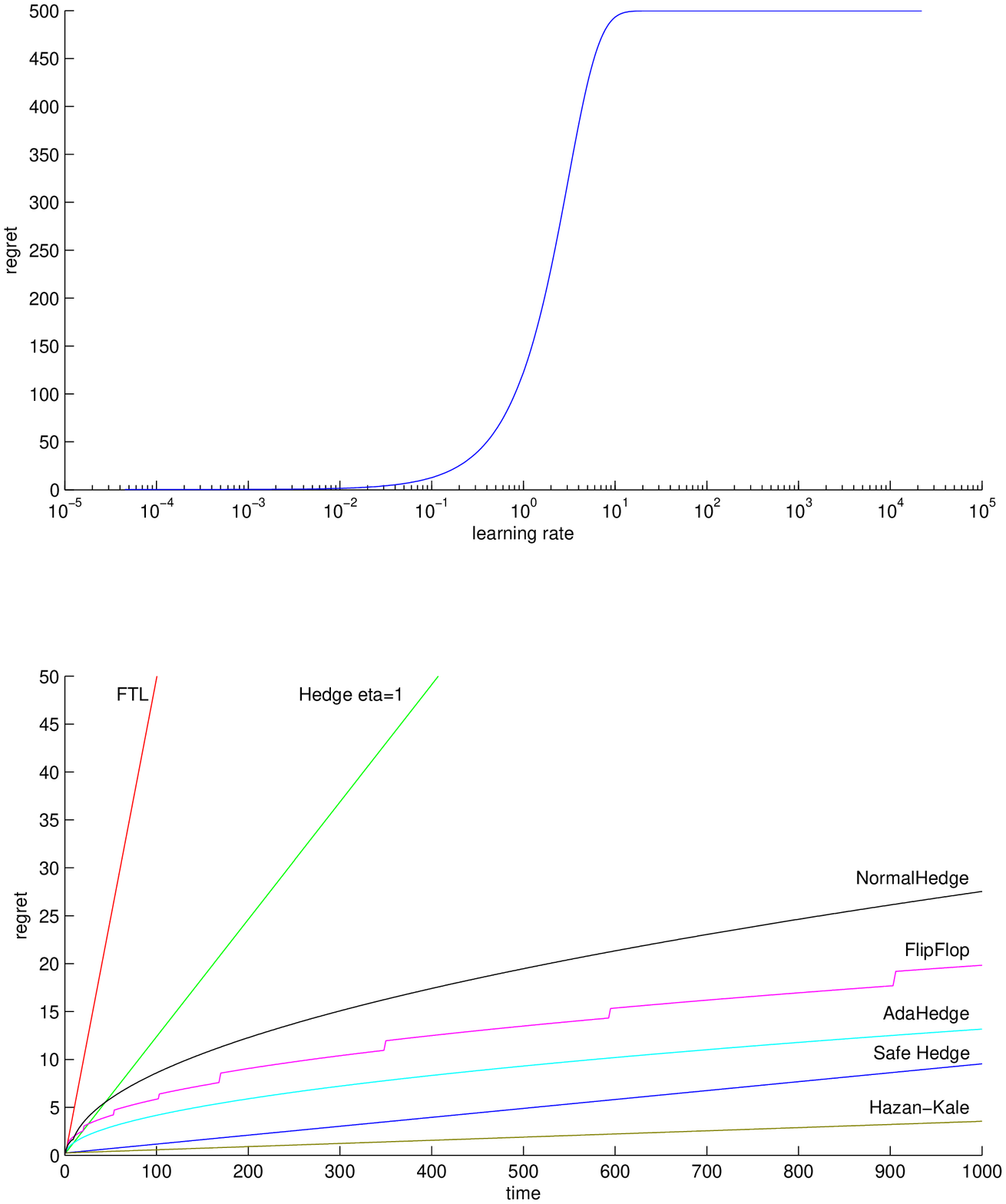}}
\caption{Hedge regret for data set 1 (FTL worst-case)\label{fig:exp1}}
\end{figure}

\subsection{Experiment~2. Best case for FTL}\label{sec:exp2}
The second experiment is defined by $\vloss_1=(1~0)$, and
$f_2(t)=3/2$.  The induced loss matrix $\vloss$ starts as follows:
\[\begin{pmatrix}
1&1&0&1&0&\cdots\\
0&0&1&0&1&\cdots
\end{pmatrix}^\top.
\]
These data look very similar to the first experiment, but
as the top graph in Figure~\ref{fig:exp2} illustrates, because of this
small change, it is now viable to reduce the regret by using a very
\emph{large} learning rate. In particular, since there are no leader
changes after the first round, FTL incurs a regret of only $1/2$.

As in the first experiment, the regret increases linearly in $t$ for
every fixed $\eta$ (provided it is less than $\infty$); but now the
constant of linearity is large only for learning rates close to
$1$. Once FlipFlop enters the FTL regime for the second time, it stays
there indefinitely, which results in bounded regret. We observe that
NormalHedge adapts in the same way to these data. The behaviour of the
other algorithms is very similar to the first experiment, and as a
consequence their regret grows without bound.

\begin{figure}[!ht]
\centering
\mbox{\kern-1cm\includegraphics[height=0.95\textheight]{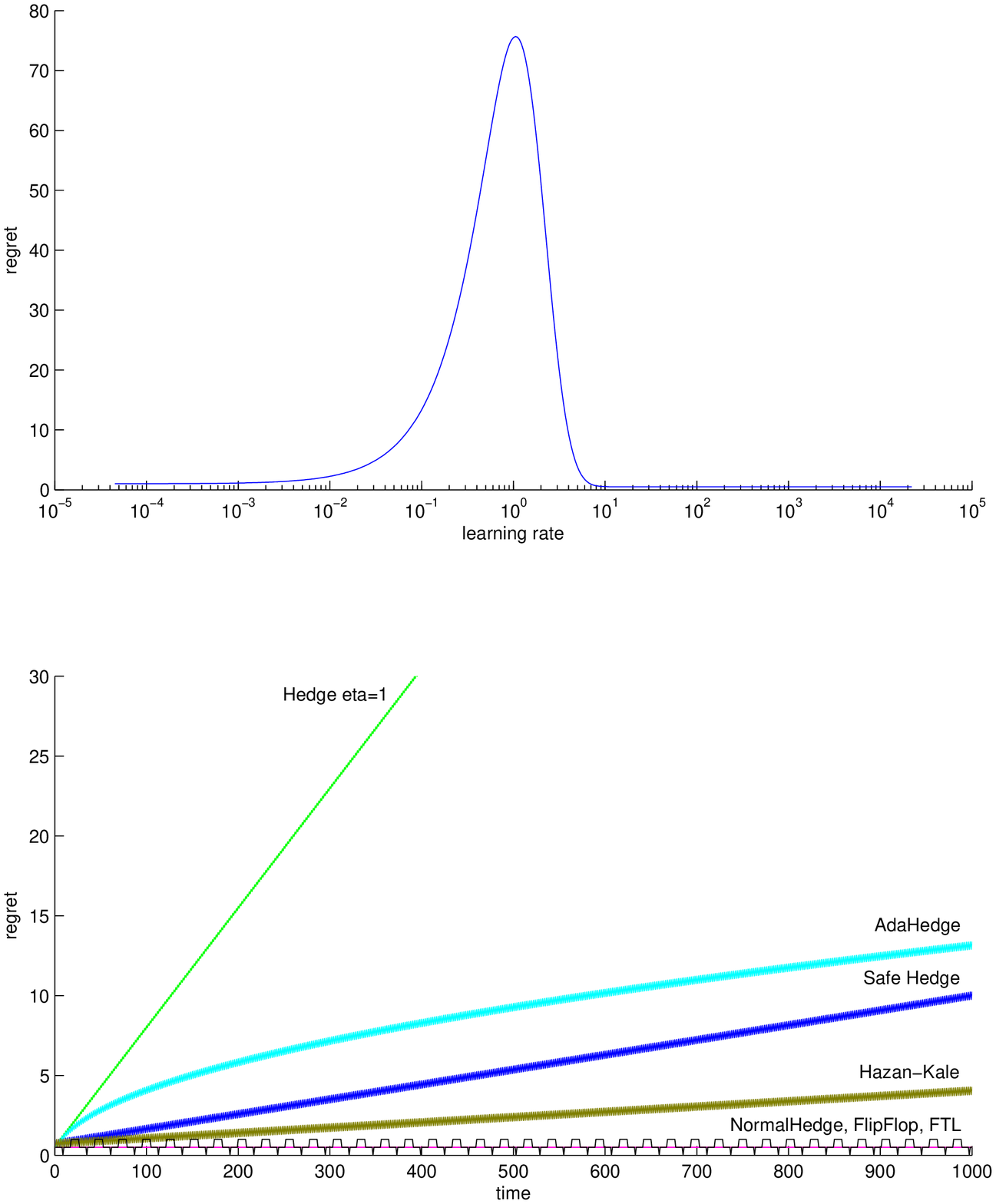}}
\caption{Hedge regret for data set 2 (FTL best-case)\label{fig:exp2}}
\end{figure}

\subsection{Experiment~3. Weights do not concentrate in AdaHedge}
The third experiment uses $\vloss_1=(1~0)$, and $f_3(t)=t^{0.4}$. The
first few loss vectors are the same as in the previous experiment, but
every now and then there are two loss vectors $(1~0)$ in a row, so
that the first expert gradually falls behind the second in terms of
performance. By $t=T=1000$, the first expert has accumulated 508 loss,
while the second expert has only 492.

For any fixed learning rate $\eta$, the weights used by Hedge now
concentrate on the second expert. We know from
Lemma~\ref{lem:bernstein} that the mixability gap in any round
$t$ is bounded by a constant times the variance of the loss under the
weights played by the algorithm; as these weights concentrate on the
second expert, this variance must go to zero. One can show that this
happens quickly enough for the cumulative mixability gap to be
\emph{bounded} for any fixed $\eta$ that does not vary with $t$ or
depend on $T$. From~\eqref{eq:hedge_regret} we have
\[
\regret_T^{(\eta)}=M_T-L_T^*+\Delta^{(\eta)}_T\le\frac{\ln K}{\eta}+\tn{bounded}=\tn{bounded}.
\]
So in this scenario, as long as the learning rate is kept fixed, we
will eventually learn the identity of the best expert. However, if the
learning rate is very small, this will happen so slowly that the
weights still have not converged by $t=1000$. Even worse, the top
graph in Figure~\ref{fig:exp3} shows that for intermediate values of
the learning rate, not only do the weights fail to converge on the
second expert sufficiently quickly, but they are sensitive enough to
increase the overhead incurred each round.

For this experiment, it really pays to use a large learning rate
rather than a safe small one. Thus FTL, Hedge with $\eta=1$, FlipFlop
and NormalHedge perform excellently, while safe Hedge, AdaHedge and
Hazan and Kale's algorithm incur a substantial overhead. Extrapolating
the trend in the graph, it appears that the overhead of these
algorithms is \emph{not} bounded. This is possible because the three
algorithms with poor performance use a learning rate that decreases as
a function of $t$. As a concequence the used learning rate may remain
too small for the weights to concentrate. For the case of AdaHedge,
this is an example of the ``nasty feedback loop'' described in
Section~\ref{sec:flipflop}.

\begin{figure}[!ht]
\mbox{\kern-1cm\includegraphics[height=0.95\textheight]{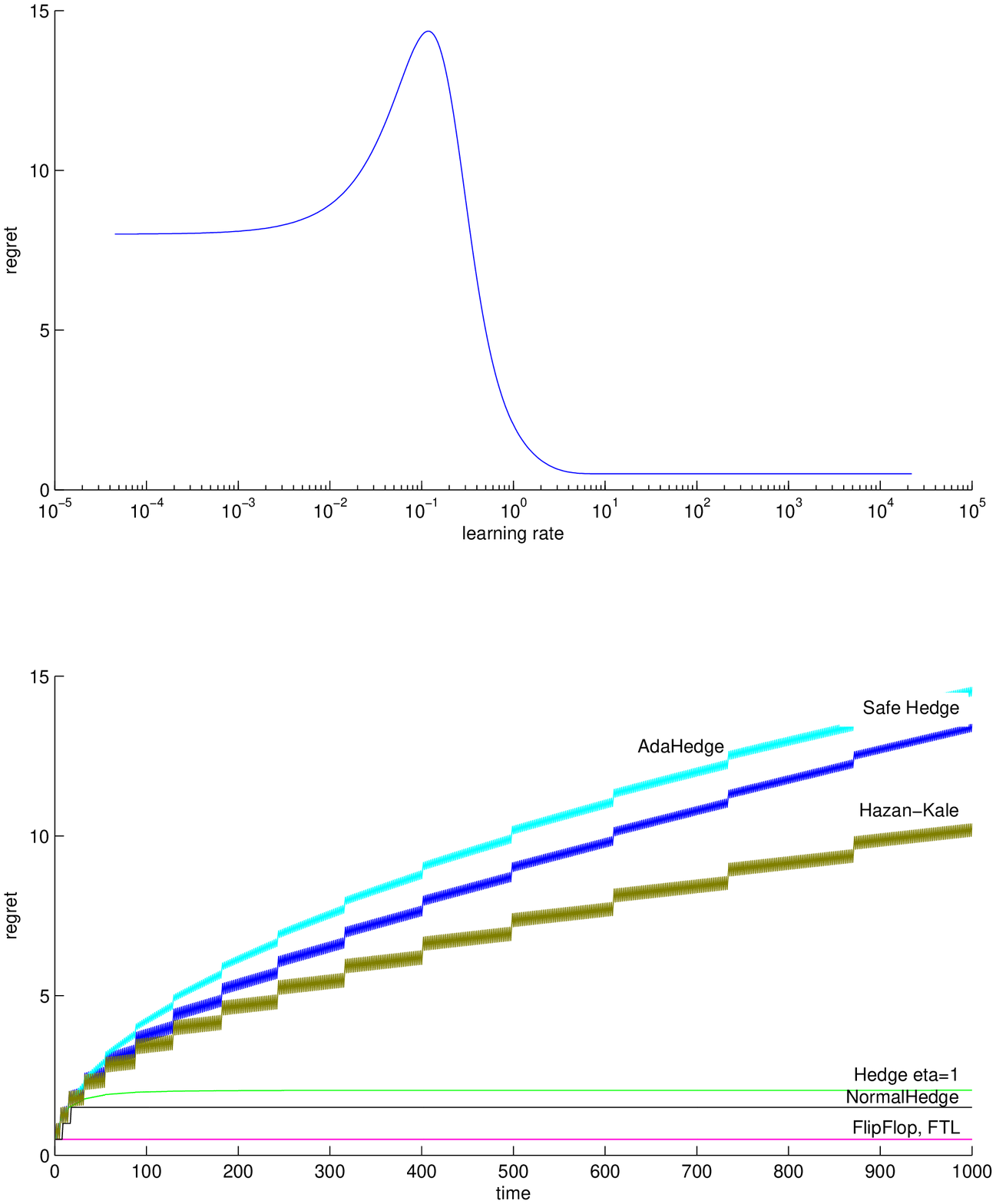}}
\caption{Hedge regret for data set 3 (weights do not concentrate in AdaHedge)\label{fig:exp3}}
\end{figure}

\subsection{Experiment~4. Weights do concentrate in AdaHedge}\label{sec:exp4}
The fourth and last experiment uses $\vloss_1=(1~0)$, and
$f_4(t)=t^{0.6}$. The losses are comparable to those of the third
experiment, but the performance gap between the two experts is
somewhat larger. By $t=T=1000$, the two experts have loss 532 and 468,
respectively. It is now so easy to determine which of the experts is
better that the top graph in Figure~\ref{fig:exp4} is nonincreasing:
the larger the learning rate, the better.

The algorithms that managed to keep their regret bounded in the
previous experiment obviously still perform very well, but it is
clearly visible that AdaHedge now achieves the same. As discussed
below Theorem~\ref{thm:variance_bound}, this happens because the
weight concentrates on the second expert quickly enough that
AdaHedge's regret is bounded in this setting. Thus, while the previous
experiment shows that AdaHedge can be tricked into reducing the
learning rate while it would be better not to do so, the present
experiment shows that on the other hand, sometimes AdaHedge does adapt
really nicely to easy data, in contrast to algorithms that are tuned
in terms of a worst-case bound.

\begin{figure}[!ht]
\mbox{\kern-1cm\includegraphics[height=0.95\textheight]{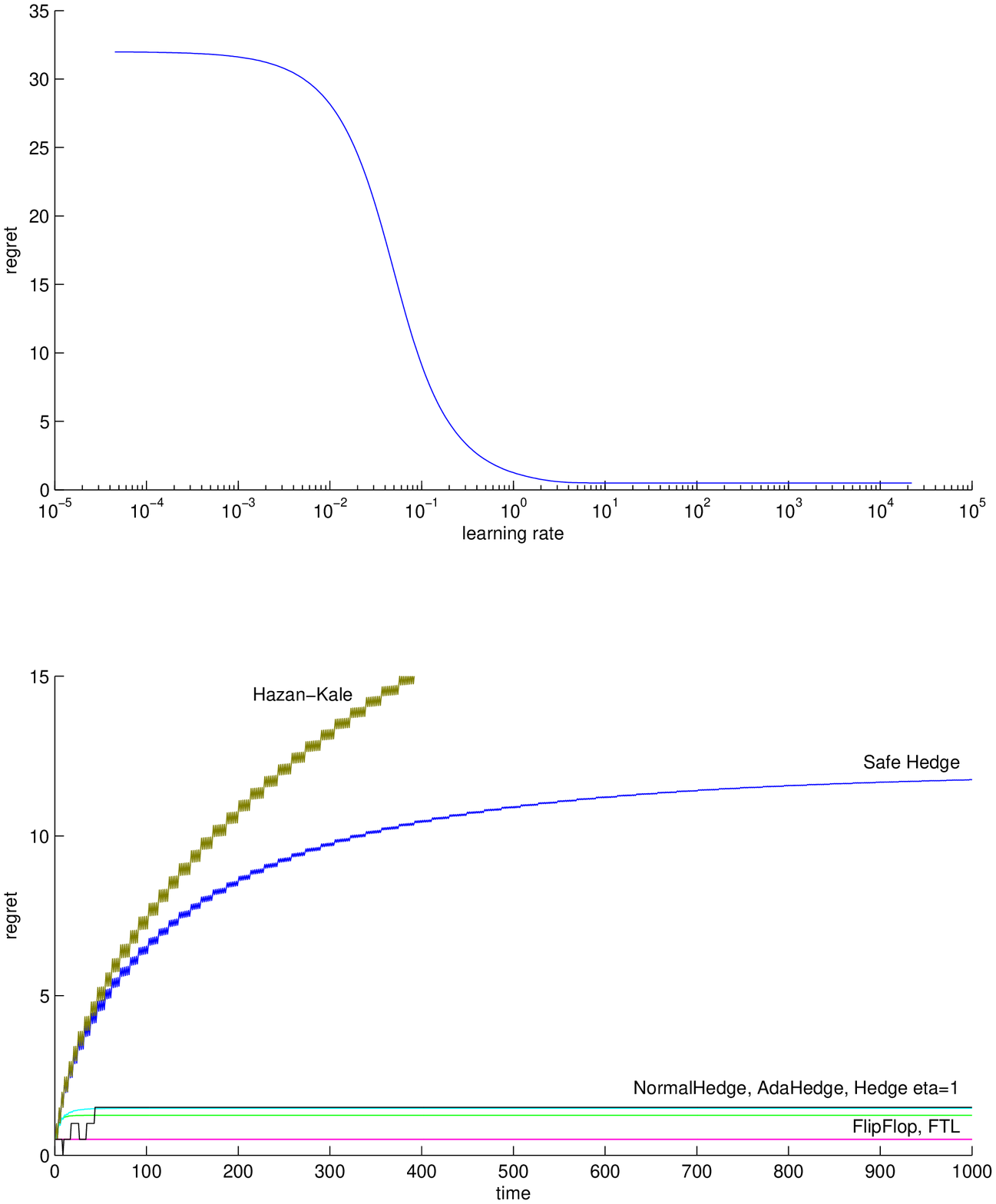}}
\caption{Hedge regret for data set 4 (weights concentrate in AdaHedge)\label{fig:exp4}}
\end{figure}

\clearpage
\section{Discussion and Conclusion}\label{sec:discussion}
The main contributions of this work are twofold. First, we develop a
new hedging algorithm called AdaHedge. The analysis simplifies
existing results and we obtain improved bounds
(Theorems~\ref{thm:variance_bound} and \ref{thm:adahedge_regret}).
Moreover, AdaHedge is the first sophisticated Hedge algorithm that is
``fundamental'', i.e. its weights are invariant under translation and
scaling of the losses (Section~\ref{sec:invariance}).  Second, we
explain in detail why it is difficult to tune the learning rate such
that good performance is obtained both for easy and for hard data, and
we address the issue by developing the FlipFlop algorithm.  FlipFlop
never performs much worse than the Follow-the-Leader strategy, which
works very well on easy data (Lemma~\ref{lem:ftl_regret}), but it also
retains a worst-case bound similar to the bound for AdaHedge
(Theorem~\ref{thm:flipflop_regret}).  As such, this work may be seen
as solving a special case of a more general question. Below we briefly
address this question and then place this work in a broader context,
which provides an ambitious agenda for future work.

\subsection{General Question: Competing with Hedge for any fixed
  learning rate}
FlipFlop has regret to within a multiplicative constant of Hedge with
learning rate $\infty$ (FTL) and Hedge with a variable, nonincreasing
learning rate which achieves optimal regret in the worst-case. It is
now natural to ask whether we can design a ``Universal Hedge''
algorithm that can compete with Hedge with any fixed learning rate $0
< \eta \leq \infty$. That is, for all $T$, the regret up to time $T$
of Universal Hedge should be within a constant factor $C$ of the
regret incurred by Hedge run with the fixed $\hat{\eta}$ that
minimizes the Hedge loss $H^{(\hat\eta)}_T$. This appears to be a
difficult question, and maybe such an algorithm does not even
exist. Yet even partial results (such as an algorithm that competes
with $\eta \in [ \sqrt{\ln (K)/T}, \infty]$ or with a factor $C$ that
increases slowly, say, logarithmically, in $T$) would already be of
significant interest.

In this regard, it is interesting to note that in practical
applications, the learning rates chosen by sophisticated versions of
Hedge do not always perform very well; higher learning rates often do
better.
This is noted by \citet{DevaineGGS12}, who resolve this issue
by adapting the learning rate sequentially in an ad-hoc fashion which
works well in their application, but for which they can provide no
guarantees. A Universal Hedge algorithm would adapt to the optimal
learning rate-with-hindsight. FlipFlop is a first step in this
direction. Indeed, it already has some of the properties of such an
ideal algorithm: under some conditions we can show that if Hedge
achieves bounded regret using \emph{any} learning rate, then FTL, and
therefore FlipFlop, also achieves bounded regret:
\begin{theorem}\label{thm:ftl_also_has_bounded_regret}
Fix any $\eta \ge 0$. For $K=2$ experts with losses in $\{0,1\}$ we have
\[\regret_T^{(\eta)} \text{ is bounded}
~\Rightarrow~
\Rftl_T \text{ is bounded}
~\Rightarrow~
\Rff_T \text{ is bounded.}\]
\end{theorem}
The proof is in
Appendix~\ref{proof:ftl_also_has_bounded_regret}. While the second
implication remains valid for more experts and other
losses, we  currently do not know if the first implication continues
to hold as well.

\subsection{The Big Picture}\label{sec:big}
Broadly speaking, a ``learning rate'' is any single scalar parameter
controlling the relative weight of the data and a prior regularization
term in a learning task.  Such learning rates pop up in batch settings
as diverse as $L_1/L_2$-regularized regression such as Lasso and
Ridge, standard Bayesian nonparametric and PAC-Bayesian inference
\citep{Zhang06b,Audibert04,Catoni07}, and --- as in this paper --- in
sequential prediction. In batch settings one may sometimes set the
learning rate by cross-validation, but this does not always come with
theoretical guarantees, and cannot easily be extended to the sequential
prediction
setting. In a Bayesian approach, one can set the learning rate by
treating it as just another parameter, equipping it with a prior and
marginalizing or determining the MAP value; it is known that this can
fail dramatically however, if all the models under consideration are
wrong \citep{Grunwald2012}.  All the applications just mentioned are
similar in that they can formally be seen as variants of Bayesian
inference --- Bayesian MAP in Lasso and Ridge, randomized drawing from
the posterior (``Gibbs sampling'') in the PAC-Bayesian setting and
Hedge in the setting of this paper. An ideal method for adapting the
learning rate would work in all such cases.  We currently have methods
that are guaranteed to work for a few special cases (see
Table~\ref{tab:CompeteEta}). It is encouraging that all these methods are based on the
same, apparently fundamental, quantity, the {\em mixability gap\/} as
defined before Lemma~\ref{lem:basic_properties}: they all employ
different techniques to ensure a learning rate under which the
posterior is concentrated and hence the mixability gap is small. This
gives some hope that the approach can be taken even further.
{ \begin{table}[h]\small
\centerline{%
\begin{tabular}{|l|c|c|l|l|l|c|}
\hline method & mode &  complexity & setting & minimizes & 
\parbox{2.1 cm}{competes\;with best $\eta$ in:}  & \parbox{1.5 cm}{\tabelT predicts/ estimates \tabelB}   \\ \hline
FlipFlop & sequential & finite & worst-case & regret & 
$ \tabelT \tabelB\eta \in \{\eta^\tn{flop}_t ,\infty
\}$ & averages \\
& prediction &&&&&\\
\hline
\parbox{2.9 cm}{\tabelT \mbox{safe two-part MDL} \citep{Grunwald2011}\tabelB}   & batch &
\parbox{1.5 cm}{countably infinite} & \parbox{1.4 cm}{stochastic, i.i.d.} & \parbox{1.5 cm}{excess risk} & 
$ \eta \in {\mathbb B}_2$ & point \\
\hline
\parbox{2.8 cm}{\tabelT \mbox{safe Bayes} \citep{Grunwald2012} \tabelB}   & batch &
\parbox{1.5 cm}{completely arbitrary} &\parbox{1.4 cm}{stochastic, i.i.d.} & \parbox{1.5 cm}{excess risk} & 
$ \eta \in {\mathbb B}_2$ & averages \\
\hline
\end{tabular}}
\caption{Methods that compete with the best $\eta$ for special cases}
\label{tab:CompeteEta}
\end{table}}

\noindent In Table~\ref{tab:CompeteEta}, ``complexity'' refers to the
maximum number of actions/experts in the DTOL setting of FlipFlop and
the maximum number of predictors (e.g.\ classifiers, regression
functions) with prior support in the stochastic setting. In the
stochastic setting we invariably assume that data are of the form
$(X_i,Y_i)$ and the goal is to predict $Y$ based on $X$.  ${\mathbb
  B}_2$ is defined as the set $\{1, 2^{-1}, 2^{-2}, \ldots \}$. The
Safe Bayes and MDL algorithms may even be capable of competing with
the best $\eta \in (0, \infty)$.  While we currently do not know
whether this is the case, we note that, in the stochastic setting,
being able to compete with the best $\eta \in{\mathbb B}_2$ is already
satisfactory: it implies that one can achieve minimax optimal risk
convergence rates in a variety of settings, e.g. if a Tsybakov margin
condition holds \citep{Grunwald2012}.

The safe two-part MDL estimator produces point estimates of the best
available predictors; analogously to FlipFlop, the safe Bayesian
estimator averages all predictors according to its posterior. The two
``safe'' algorithms can deal with arbitrary loss functions as long as
the loss is almost surely bounded. If the data are sampled from a
distribution with bounded support, this even holds if the loss
function is itself unbounded.

All this suggests a major goal for future work: extending the
worst-case approach of this paper to the settings that are currently
dealt with only in the stochastic case. First, as already explained
above, we would like to be able to compete with all $\eta$ in some set
that contains a whole range rather than just two values. Second, we
would like to compete with the best $\eta$ in a setting with a
countably infinite number of experts equipped with an arbitrary prior
mass function $\w_1$. Third,
as an ultimate goal, we would like to develop a method that can
compete with the best $\eta$ with completely arbitrary sets of
experts equipped with some prior distribution $W$.
The second and third goal require a slight modification of the type
of results in this paper: currently, our results all start with the basic
identity and bound (\ref{eq:wcbound}), repeated here for convenience:
$$\regret_T = (M_T-L_T^*) +\Delta_T \le \frac{\ln K}{\eta} + \Delta_T.
$$
For the case of infinitely many experts, this should be replaced by
the following identity and inequality, which hold simultaneously for
all distributions $Q$ on the set of experts; the idea is to choose $Q$
so as to get a useful bound.
\begin{align}
H_T - Q \dot L_T &= (M_T - Q \dot L_T) + \Delta_T \notag \\
\label{eq:wcboundb} & = \left( \; \inf_{V} \left\{ \frac{D(V \|
      W_1)}{\eta} + V \dot L_T \right\} - Q \dot L_T \; \right) +
\Delta_T \notag \\ &\leq \frac{D(Q \| W_1)}{\eta} + \Delta_T,
\end{align}
where for convenience we defined $Q \dot L_T := {\mathbf E}_{K \sim
  Q}[L_{T,K}]$, the expected value of the cumulative loss under
distribution $Q$. Here $W_1$ is a user-defined prior distribution on the set of
experts, analogous to our probability mass function $\w_1$, and $D(\cdot
\| \cdot)$ denotes the KL divergence between two distributions on experts.
The inequality is trivial; the equality is a well-known result both in
the sequential prediction and the PAC-Bayesian literature; see e.g. \cite{Zhang06b}.
To make (\ref{eq:wcboundb}) more concrete, consider a countable set of
experts,  fix an expert $k$
and take $Q$ to be a point mass on $k$. Then $Q \dot L_T = L_{T,k}$
and $D(Q \| W_1)$ becomes equal to $- \ln w_1(k)$, so (\ref{eq:wcboundb}) can
be further rewritten as 
$$ H_T - L_{T,k} 
\leq \frac{-
  \ln w_{1,k}}{\eta} + \Delta_T, 
$$
We hope that using this bound, analogously to our use of
(\ref{eq:wcbound}) in the current paper, one can prove bounds similar
to those appearing in Theorem~\ref{thm:adahedge_regret}
and~\ref{thm:flipflop_regret}, with all occurrences of $L^*_T$ and
$\ln K$ replaced by $L_{T,k}$ and $- \ln w_{1,k}$. Here $k$ can be
thought of as a `comparator' expert, and the bounds should hold
uniformly for all $k$ but get progressively weaker for $k$ with small
initial prior weight $w_{1,k}$. For the case of uncountable sets of
experts, the hope is again to prove results similar to
Theorem~\ref{thm:adahedge_regret} and~\ref{thm:flipflop_regret}, but
now based on (\ref{eq:wcboundb}). Such results would give strong
worst-case performance bounds on huge, ``nonparametric'' sets of
experts such as Gaussian process models. Currently such worst-case
bounds exist for the logarithmic loss \citep{KakadeSF06}, but not for
any other loss function.  \acks{We would like to thank Wojciech
  Kot{\l}owski and Gilles Stoltz for critical feedback. This work was
  supported in part by the IST Programme of the European Community,
  under the PASCAL Network of Excellence, IST-2002-506778 and by NWO
  Rubicon grants 680-50-1010 and 680-50-1112.}

\clearpage
\appendix

\section{Proof of
  Lemma~\ref{lem:basic_properties}}\label{proof:basic_properties}

The result for $\eta = \infty$ follows from $\eta < \infty$ as a
limiting case, so we may assume without loss of generality that $\eta <
\infty$. Then $m_t \leq h_t$ is obtained by using
Jensen's inequality to move the logarithm inside the
expectation, and $m_t \geq 0$ and $h_t \leq 1$ follow by bounding all
losses by their minimal and maximal values, respectively. The next two
items are analogues of similar basic results in Bayesian probability.
Item~\ref{it:cumbayes} generalizes the chain rule of probability
$\Pr(x_1,\ldots,x_T) = \prod_{t=1}^T \Pr(x_t \mid x_1,\ldots,x_{t-1})$:
\[
M_T = -\frac{1}{\eta} \ln\prod_{t=1}^T
      \frac{\!\!\!\!\!\w_1 \dot e^{-\eta \L_t}}{\w_1 \dot e^{-\eta \L_{t-1}}}
    =-\frac{1}{\eta}\ln(\w_1\dot e^{-\eta\L_T}).
\]
For the third item, use item~\ref{it:cumbayes} to write
\[
M_T = -\frac{1}{\eta}\ln\left(\sum_k w_{1,k}e^{-\eta L_{T,k}}\right).
\]
The lower bound is obtained by bounding all $L_{T,k}$ from below by $L^*_T$; for
the upper bound we drop all terms in the sum except for the term
corresponding to the best expert and use $w_{1,k} = 1/K$.

For the last item, let $0 < \eta < \gamma$ be any two learning rates.
Then Jensen's inequality gives
\[
-\frac{1}{\eta}\ln\w_1\dot e^{-\eta\L_T}
=-\frac{1}{\eta}\ln
\w_1\dot\left(e^{-\gamma\L_T}\right)^{\eta/\gamma}
\ge-\frac{1}{\eta}\ln
\left(\w_1\dot e^{-\gamma\L_T}\right)^{\eta/\gamma}
=-\frac{1}{\gamma}\ln\w_1\dot e^{-\gamma\L_T}.
\]
This completes the proof.\hfill$\BlackBox$

\section{Proof of Theorem~\ref{thm:ftl_also_has_bounded_regret}}\label{proof:ftl_also_has_bounded_regret}
Suppose that FTL has unbounded regret. We argue that Hedge with fixed
$\eta$ must have unbounded regret as well. First remove all trials
where both experts suffer the same loss, as these trials do not change
the regret of either FTL or Hedge. Abbreviate $d_t = L_{t,2} -
L_{t,1}$. We say that a leader change happens at $t$ when $d_{t-1}
d_{t+1} < 0$, that is, $d_t$ crosses zero at $t$. Since FTL has
unbounded regret, there are infinitely many leader changes.

We call a point-pair $(t,t+1)$ a \emph{local extremum} if the losses in trials $t$ and $t+1$ are opposite, i.e.\ $(d_{t+1}-d_t)(d_t - d_{t-1}) < 0$. Observe that a leader change can not be a local extremum. Over a local extremum, Hedge suffers loss $>1$ but the best expert only suffers loss $1$. The regret of Hedge is hence decreased when the trials $t$ and $t+1$ are removed. Iterated removal of local extrema leads to the $d_t$ sequence
\[
0, +1, 0, -1, 0, +1, 0, -1, \ldots
\]
The regret of Hedge on this sequence is linear in $t$. To see
this, observe that over one period the loss of the best expert
increases by $2$, while the loss of Hedge increases by
\[
2 \frac{1}{2} + 2 \frac{1}{1+e^{-\eta}}  ~>~ 2.
\]
Hence the Hedge regret is unbounded on the original loss sequence.\hfill$\BlackBox$

\bibliography{flipflop}

\end{document}